\documentclass[journal]{IEEEtran}
\usepackage{amsmath,amsfonts}
\usepackage{algorithmic}
\usepackage{array}

\usepackage{textcomp}
\usepackage{upgreek}
\usepackage{stfloats}
\usepackage{url}
\usepackage{multirow}
\usepackage{makecell}
\usepackage{fmtcount}
\usepackage{makecell,multirow}  
\usepackage{multicol}
\usepackage{verbatim}
\usepackage{bbm}
\usepackage{graphicx}
\usepackage{cite}
\usepackage{amsmath}
\usepackage[ruled,vlined]{algorithm2e}
\usepackage{amssymb}
\usepackage{mathtools}
\usepackage{amsthm}
\usepackage{subfigure}
\usepackage[capitalize,noabbrev]{cleveref}
\newtheorem{assumption}{Assumption}
\theoremstyle{plain}
\newtheorem{theorem}{Theorem}

\newtheorem{lemma}{Lemma}

\theoremstyle{definition}

\theoremstyle{remark}
\newtheorem{remark}[]{Remark}
\newcounter{scenario}
\usepackage{siunitx}
\makeatletter

\newcommand{\Rmnum}[1]{\expandafter\@slowromancap\romannumeral #1@}
\makeatother

\begin{document}
\renewcommand{\qedsymbol}{}

\title{Heterogeneous Multi-Agent Reinforcement Learning for Distributed Channel Access in WLANs}

\author{Jiaming Yu,~\IEEEmembership{Graduate Student Member,~IEEE,} Le Liang,~\IEEEmembership{Member,~IEEE,} Chongtao Guo,~\IEEEmembership{Member,~IEEE,} \\Ziyang Guo,~\IEEEmembership{Member,~IEEE,} Shi Jin,~\IEEEmembership{Fellow,~IEEE,} and Geoffrey Ye Li,~\IEEEmembership{Fellow,~IEEE}
\thanks{This article was presented in part at the IEEE Wireless Communication Networking Conference (WCNC), Milan, Italy, March 2025.}
\thanks{Jiaming Yu, Le Liang, and Shi Jin are with the National Mobile Communications Research Laboratory, Southeast University, Nanjing 210096, China
(e-mail: jiaming@seu.edu.cn; lliang@seu.edu.cn; jinshi@seu.edu.cn). Le Liang is also with Purple Mountain Laboratories, Nanjing 211111, China.}
\thanks{Chongtao Guo is with the College of Electronics and Information Engineering, Shenzhen University, Shenzhen 518060, China (e-mail: ctguo@szu.edu.cn).}
\thanks{Ziyang Guo is with the Wireless Technology Lab, 2012 Laboratories, Huawei Technologies Company Ltd., Shenzhen 518129, China (e-mail: guoziyang@huawei.com).}
\thanks{Geoffrey Ye Li is with the ITP Lab, the Department of Electrical and Electronic Engineering, Imperial College London, SW7 2BX London, U.K. (e-mail: geoffrey.li@imperial.ac.uk).}
}

\maketitle

\begin{abstract}
This paper investigates the use of multi-agent reinforcement learning (MARL) to address distributed channel access in wireless local area networks. In particular, we consider the challenging yet more practical case where the agents heterogeneously adopt value-based or policy-based reinforcement learning algorithms to train the model. 
We propose a heterogeneous MARL training framework, named QPMIX, which adopts a centralized training with distributed execution paradigm to enable heterogeneous agents to collaborate. Moreover, we theoretically prove the convergence of the proposed heterogeneous MARL method when using the linear value function approximation. Our method maximizes the network throughput and ensures fairness among stations, therefore, enhancing the overall network performance. Simulation results demonstrate that the proposed QPMIX algorithm improves throughput, mean delay, delay jitter, and collision rates compared with conventional carrier-sense multiple access with collision avoidance (CSMA/CA) mechanism in the saturated traffic scenario. 
Furthermore, the QPMIX algorithm is robust in unsaturated and delay-sensitive traffic scenarios. It coexists well with the conventional CSMA/CA mechanism and promotes cooperation among heterogeneous agents.
\end{abstract}


\section{Introduction}
\IEEEPARstart {W}{i-Fi} has become an important technology in wireless local area networks (WLANs), and has been widely used in personal, home and enterprise environments, with Wi-Fi 6 \cite{khorov2018tutorial} and Wi-Fi 7 \cite{deng2020ieee} as the latest commercially available technology. 
However, achieving seamless connectivity across diverse devices and applications, such as virtual reality/augmented reality, remote surgery, and online gaming, remains a substantial challenge for Wi-Fi 6 and 7.
These applications require higher data rates, lower latency, and higher reliability than the current Wi-Fi technology can provide. As a result, developing the next generation of Wi-Fi technology, i.e., Wi-Fi 8, has attracted increasing attention from both industry and academia \cite{galati2024will}.

To meet the demands of the aforementioned applications, it is crucial not only to enhance the physical layer rate but also to improve the throughput of the media access control (MAC) layer. A fundamental aspect of MAC layer design is distributed channel access (DCA), where multiple users utilize a shared channel in a fully decentralized manner without any centralized scheduling mechanism.
One of the most popular DCA schemes is the carrier-sense multiple access with collision avoidance (CSMA/CA) \cite{colvin1983csma}, which utilizes random access to avoid collisions. However, it exhibits unsatisfactory throughput performance due to its reliance on the binary exponential back-off (BEB) algorithm. In this algorithm, stations (STAs) are required to wait for a random back-off period determined by the contention window (CW) size before initiating a transmission. Furthermore, the CW size doubles following each collision, which can lead to increased latency and reduced efficiency. To address these issues, enhancements to the BEB mechanism have been proposed in prior research works \cite{heusse2005idle,magistretti2011wifi,barcelo2011towards,misra2014semi}. In \cite{barcelo2011towards}, a deterministic BEB mechanism applied twice consecutively after successful transmissions is introduced to avoid collisions in WLANs, particularly in lossy channels and scenarios with high frame error rates. Additionally, a semi-distributed BEB algorithm has been proposed to further reduce frame collisions \cite{misra2014semi}. These approaches aim to improve the efficiency of CSMA/CA by mitigating some of the limitations of the traditional BEB mechanism while maintaining simplicity in implementation. 
Nevertheless, these methods are still inadequate. 
On the one hand, due to the nonlinear impact of parameter dependencies on network performance, it is often necessary to jointly optimize multiple parameters, which is a challenging task \cite{wilhelmi2021spatial}. On the other hand, the increasing diversity of service requirements raised by new applications, such as ultra-low latency and ultra-high reliability, increases the complexity of conventional DCA mechanism and cannot be easily satisfied by the traditional CSMA/CA mechanism.

To overcome the limitations of traditional methods, a multitude of studies have utilized artificial intelligence (AI) techniques to optimize parameters on top of the CSMA/CA framework to improve the throughput performance.
For example, a supervised-learning-based scheme using a random forest approach has been proposed to dynamically adjust CW values, effectively reducing collisions and idle periods while significantly improving throughput, latency, and fairness \cite{abyaneh2019intelligent}. 
Based on machine learning, a simple yet effective method based on the fixed-share algorithm \cite{herbster1998tracking} is developed in \cite{edalat2019dynamically} to optimize CW values, which utilizes a modified fixed-share algorithm to adjust CW values based on current and past network contention conditions, thereby improving channel utilization. 
Furthermore, recent advancements in reinforcement learning (RL) provide a more intelligent and efficient solution for future wireless communication systems \cite{liang2019deep, qin2024ai, liang2019spectrum}.
In particular, an RL-based approach has been proposed in \cite{wydmanski2021contention} to learn the appropriate CW value under varying network conditions. In this approach, the access point (AP) acts as an agent to control the CW values of STAs. Moreover, a learning algorithm that utilizes the post-decision state \cite{amuru2015send} can expedite the learning process by leveraging the past knowledge of system components such as CW and the transmission buffer occupancy. 
Building upon these advancements, a hierarchical RL approach has been proposed \cite{huang2023hierarchical}, where the AP learns the optimal transmit power at the high-level policy and the clear channel assessment (CCA) threshold at the low-level policy. 

In addition to optimizing parameters on top of the CSMA/CA framework, recent studies have explored novel channel access schemes based on single agent reinforcement learning (SARL) and multi-agent reinforcement learning (MARL). For example, a SARL-based multiple access algorithm has been proposed in \cite{yu2020non} to maximize network the throughput while keeping $\alpha$ fairness among different users. Expanding upon SARL, MARL-based DCA algorithms have gained attention. However, current MARL-based DCA methods predominantly rely on value-based algorithms. For instance, 
in \cite{zhang2020enhancing}, each STA operates as a deep Q-network (DQN) agent \cite{mnih2015human} to make transmission or waiting decisions based on historical channel states and feedback from the AP. Additionally, a federated learning framework is utilized to aggregate models, therefore, ensuring fairness among users.
A similar network coexistence problem with imperfect feedback channels has been addressed by an MARL algorithm in \cite{yu2021multi}, which utilizes a feedback recovery mechanism to retrieve the missing feedback information. It employs a two-stage action selection mechanism to facilitate the coherent decision-making, thereby mitigating transmission collisions among the agents. Furthermore, a distributed multichannel access protocol based on dueling Q-network architecture has been proposed in \cite{sohaib2021dynamic} to improve both throughput and fairness in a dynamic network environment.
Moreover, based on the well known value-based QMIX algorithm \cite{rashid2020monotonic}, a MAC protocol, named QMIX-advanced listen-before-talk (QLBT), has been developed in \cite{guo2022multi}.
In the QLBT protocol, each STA is considered as an agent with the ability to make decisions on transmitting or not based on its local observations and historical data. 

While the aforementioned studies focusing on the value-based MARL algorithm, a few policy-based MARL algorithms have also been proposed to solve the DCA problem. In \cite{zhong2019deep}, a deep actor-critic RL framework has been proposed for dynamic multichannel access in both single-user and multi-user scenarios, which addresses the challenges of partial observability and channel uncertainty. Additionally, inspired by mean-field games theory, a novel mean-field-based MARL algorithm for uplink orthogonal frequency-division multiple access in IEEE 802.11ax networks has been proposed to learn the optimal channel access strategy, thereby enhancing network performance in high-density environments \cite{han2024multi}. 


The aforementioned MARL-based DCA algorithms presuppose that all nodes employ the same type of RL algorithm, i.e., either value-based or policy-based algorithms. 
However, given the dynamic nature of Wi-Fi and the variety of RL algorithms, there may exist both value-based and policy-based MARL agents within a basic service set (BSS), which we refer to as the heterogeneous MARL scenario. 
Note that the heterogeneous MARL we consider refers to algorithmic heterogeneity among STAs, rather than reward function or policy parameter heterogeneity as considered in other works \cite{mondal2022approximation}. For example, devices from different manufacturers are likely to employ different types of RL algorithms. 
Therefore, this naturally leads to heterogeneous MARL scenarios, where value-based and policy-based agents might coexist.
Unique challenges exist in such heterogeneous scenarios since many existing techniques to stablize MARL training are tailored to homogeneous MARL scenarios where all RL agents use the same type of RL algorithms, mostly value-based methods.

To solve the aforementioned heterogeneous MARL problem, this article primarily focuses on designing an effective heterogeneous MARL framework to promote seamless cooperation among STAs. To support diverse service requirements in next-generation WLANs, heterogeneous MARL simultaneously maximize the throughput for high bandwidth efficiency and ensure fairness among STAs. 
The contributions and main results of this work are summarized as follows: 
\begin{itemize}
\item We develop a novel heterogeneous MARL framework, named QPMIX for the DCA problem in next-generation WLANs. QPMIX adopts the paradigm of centralized training with distributed execution (CTDE). 
\item We theoretically analyze the convergence of the proposed heterogeneous MARL framework and characterize its convergence property when using the linear value function approximation.
\item We provide extensive simulation results to demonstrate the effectiveness of our framework in a variety of scenarios, including the saturated traffic, unsaturated traffic, and delay-sensitive traffic scenarios. We also conduct an experiment to contrast independent learning against using the QPMIX framework, and verify that QPMIX can promote better collaboration among the heterogeneous agents.
\end{itemize}
The remainder of this paper is structured as follows. The system model and RL formulation are described in Section II.
Section III provides the proposed QPMIX algorithm and protocol. The convergence of the proposed heterogeneous MARL is characterized in Section IV for the special case of linear value function approximation. Finally, simulation results are discussed in Section V, followed by the conclusion and discussion of the future work in Section VI.

\section{System Model And RL Formulation}
As illustrated in Fig. \ref{BSS}, we consider the uplink channel access problem in a time-slotted BSS where $N$ STAs attempt to communicate with an AP, known as the DCA problem. For simplicity, we assume that there are no hidden nodes, implying that each STA performs carrier sensing before accessing the channel and they can listen to one another. 
Each STA is equipped with a finite queue buffer and determines whether to access the channel based on its own observations. The packets arrive at the STA according to a specified distribution while they would be discarded when the buffer is full. After a successful transmission, the AP sends an acknowledgment (ACK) signal to the corresponding STA. The STA decides whether to send packets only when it detects that the channel is idle, which is known as the listen-before-talk mechanism in CSMA/CA. Once the transmission begins, the channel will convert to the busy state, and other STAs will wait until the end of transmission for attempting channel access. But if two or more STAs sense the idle channel in the previous time slot and decide to transmit in the current time slot, then they will not be aware of the collision until they fail to receive the ACK signal from the AP in time. 

The objective of the DCA problem is to maximize the network throughput while ensure fairness among STAs. 
Then how to coordinate access of multiple STAs such that they neither occupy the channel selfishly nor keep waiting due to collisions remains challenging. 
Fortunately, MARL algorithms have proven effective in addressing distributed decision-making in similar scenarios \cite{zhang2020enhancing, sohaib2021dynamic, guo2022multi,yu2021multi}.  
However, these approaches are only applicable to value-based agents. When both value-based and policy-based agents are present in a BSS, existing algorithms are not able to train the heterogeneous agents effectively to achieve cooperation, which is a challenging yet more practical case. Motivated by this, our goal is to design a heterogeneous MARL-based DCA algorithm so that both value-based and policy-based agents can maintain fair access to the channel and maximize the network throughput. To this end, we first introduce preliminaries on RL before presenting the underlying heterogeneous decentralized partially observable Markov decision process (Dec-POMDP) model and our RL elements design for the DCA problem. 
\begin{figure}[!t]
\centering
\includegraphics[width=3.4in]{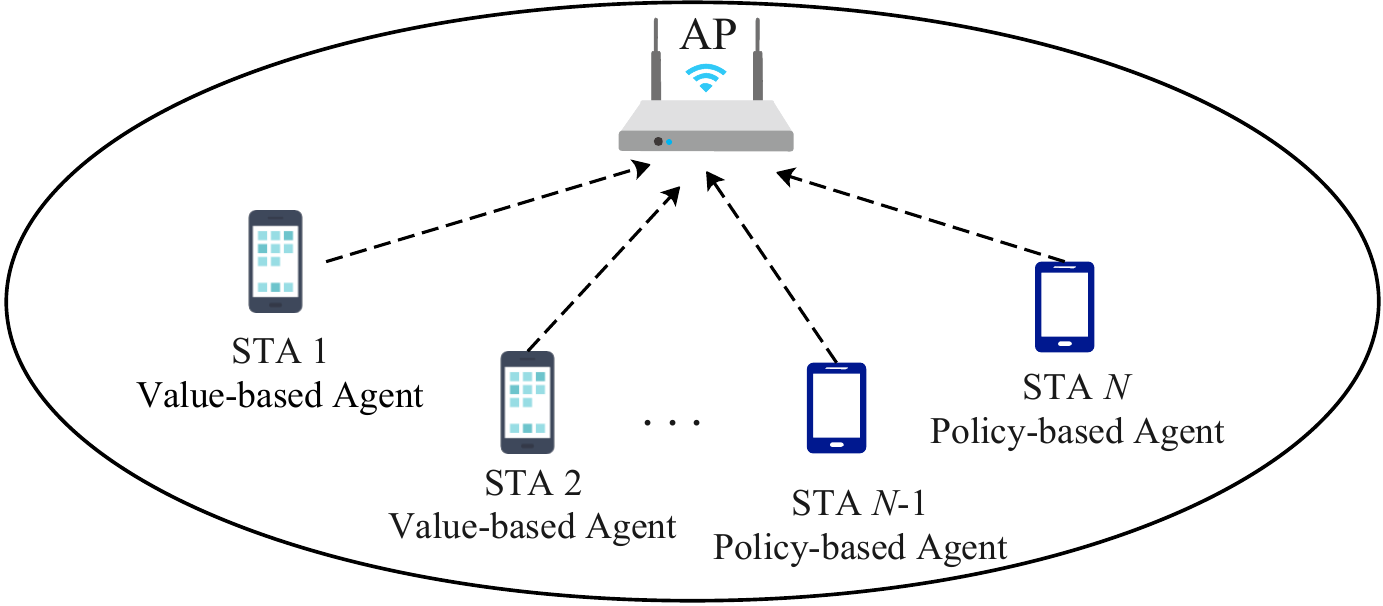}
\caption{Scenario of a BSS: multiple STAs share one common channel and determine data transmission in a distributed manner.}
\label{BSS}
\end{figure}

\subsection{Single Agent Reinforcement Learning}
In SARL problems, an agent continuously interacts with the environment and  updates its policy. At each time step $t$, the agent obtains its observation of the current state of environment $s_t$ and chooses an action $a_t$ according to its policy $\pi$. The environment then responds to the agent with a reward $r_{t}$ and moves to the next state $s_{t+1}$ with probability $P(s_{t+1}|s_t,a_t)$. The objective of the agent is to update its policy based on its historical experience to maximize the expected discounted return $\mathbb{E}[\sum_{t=0}^\infty\gamma^tr_t],$ where $\gamma \in (0, 1]$ is a discount factor. 

To solve this problem, a variety of value-based algorithms and policy-based algorithms have been proposed. Value-based algorithms learn and optimize the action-value function, i.e., the Q-function
$Q(s,a)=\mathbb{E}_\pi[\sum_{k=0}^T\gamma^kr_{t+k}|s_t=s,a_t=a]$. The optimal Q-function is progressively obtained by updating the Q-function using the Bellman equation,
\begin{equation}
Q^*(s,a)=\mathbb{E}[r_t+\gamma\max_{a'}Q^*(s_{t+1},a')|s_t=s,a_t=a].
\end{equation}
Then the agent obtains the optimal policy $\pi^*$ indirectly through the Q-function, i.e., $\pi^*(s)=\arg\max_aQ^*(s,a).$
DQN \cite{mnih2015human} is a classical value-based algorithm that employs a neural network to approximate the Q-function, thereby addressing the curse of dimensionality faced by tabular methods. 

The value-based algorithm derives the optimal policy by estimating the value function, while the policy-based algorithm learns the optimal policy directly. As an agent interacts with its environment, it records a sequence of trajectories. Given the agent policy $\pi$, the probability of a specific trajectory $\tau=\{s_{0},a_{0},r_{0},s_{1},a_{1},r_{1},...,s_{T},a_{T},r_{T}\}$ can be computed as
\begin{equation}
P(\tau)=p(s_0)\prod_{t=0}^T\pi(a_t\mid s_t)P(s_{t+1}\mid s_t,a_t),
\end{equation}
and the expected discounted return is given by
\begin{equation}
\mathbb{E}_\pi[\sum_{t=0}^T\gamma^tr_t]=\sum_\tau R(\tau)P(\tau),
\end{equation}

where $R(\tau)=\sum_{t=0}^T\gamma^tr_t$. The policy is improved by updating the policy parameters using gradient ascent to maximize the expected return. The actor-critic algorithms are an important family of policy-based algorithms, where the actor selects actions and the critic estimates the expected return. Proximal policy optimization (PPO)
\cite{schulman2017proximal} is a well-known actor-critic algorithm that utilizes the importance sampling to improve the data efficiency and an advantage function to estimate advantage value of each action.

\subsection{Multi-Agent Reinforcement Learning}

In MARL problems, each agent independently selects its actions based on its local observations of the environment to maximize the overall expected reward. Different from SARL, the main challenge of MARL is the non-stationarity and partial observability of the environment.
A key factor driving the growth of MARL algorithms is the adoption of the CTDE paradigm \cite{lowe2017multi} to  mitigates 
these challenges.
During the centralized training phase, agents have access to global information, while during execution, they rely only on local observations to make decisions \cite{oliehoek2008optimal, kraemer2016multi}.
Similar to SARL algorithm, MARL algorithms can also be divided into value-based and policy-based algorithms.
Among value-based algorithms, the value decomposition method \cite{wang2020qplex, son2019qtran} is a popular algorithm that adopts CTDE. The basic concept of the value decomposition is to break down a complicated joint value function into individual value functions that each agent can understand and utilize as a foundation for decision-making. This decomposition adheres to the individual-global-max (IGM) principle \cite{rashid2020monotonic}, which ensures consistency between joint and individual optimal actions. QMIX \cite{rashid2020monotonic}, one of the most powerful value-based MARL value decomposition algorithms, employs a hypernetwork to aggregate individual Q-value into a total Q-value. 

Mathematically, the underlying heterogeneous Dec-POMDP model of the DCA problem can be described by a tuple $\langle\mathcal{N},\mathcal{N}_1,\mathcal{N}_2,\mathcal{S},\left\{\mathcal{A}^i\right\}_{i\in\mathcal{N}},\mathcal{O}, \mathcal{P},R,\gamma\rangle$. Here, $\mathcal{N}$ denotes the set of agents with $|\mathcal{N}|=N$. $\mathcal{N}_1$ represents the set of value-based agents and $\mathcal{N}_2$ represents the set of policy-based agents, with $|\mathcal{N}_1| + |\mathcal{N}_2| = N$. $\mathcal{S}$ is the global state space. At each time step $t$, each agent $i \in \mathcal{N}$ selects an action $a^i_t$ from its action space $\mathcal{A}^i$ after receiving its observation $o^i_t \in \mathcal{O}$. The combined actions form a joint action $\boldsymbol{a_t} \in \mathcal{A}$. Additionally, $\mathcal{P}:\mathcal{S}\times\mathcal{A}\times\mathcal{S}\to[0,1]$ represents the state transition probability of the Markov decision process (MDP), and $R:\mathcal{S}\times\mathcal{A}\to\mathbb{R}$ is the reward function. The corresponding RL elements for the DCA problem are introduced hereafter.

\subsection{RL Elements for DCA Problem}
\textbf{Action:}
At each time step $t$, agent $i$ can take action $a^i_t\in\{\text{Transmit}, \text{Wait}\}$, where `Transmit' means the agent transmits at the current time slot while `Wait' means the agent waits for a time slot. Specifically, an agent only updates its action when the channel is sensed idle and its packet buffer is not empty; otherwise it directly takes the default waiting action. These actions are recorded as part of their trajectories. If an agent chooses to transmit, the transmission spans multiple time slots, corresponding to the packet length.

\textbf{Observation Space and Global State:}
Inspired by \cite{guo2022multi}, we set the observation of agent $i$ at time step $t$ as $o_{t}^{i}\triangleq\left\{z_{t}^{i},a_{t-1}^{i},l_{t}^{i},d_{t}^{i},d_{t}^{-i}\right\}$. $z_{t}^{i}$ represents the carrier-sensing result, where $z_{t}^{i}=0$ if the channel is sensed idle; otherwise $z_{t}^{i}=1$.
The variable $a^i_{t-1}$ denotes the action of agent $i$ at time step $t-1$, and the variable $l^i_t$ represents the number of time slots that the same $(z^i_t, a^i_{t-1})$ pair lasts to avoid storing a large amount of duplicated $(z^i_t, a^i_{t-1})$. Each agent maintains two counters $v^i_t$ and $v^{-i}_t$, which represent the number of time slots since the last successful transmission of agent $i$ and any other agent, respectively. When agent $i$ receives its own ACK frame at slot $t$ replied by the AP, $v^i_t$ is set to 0. If agent $i$ receives the ACK frame at slot $t$ sent by the AP to another agent, $v^{-i}_t$ is set to 0. Otherwise, we have $v^i_t = v^i_{t-1} + 1$ and $v^{-i}_t = v^{-i}_{t-1} + 1$. 
Note that ACKs are short frames and only the MAC address needs to be decoded to determine whether the ACK is sent to agent $i$. Therefore, the resulting computational and energy overhead is small.
Furthermore, we define $d_{t}^{i}\triangleq\frac{v_{t}^{i}}{v_{t}^{i}+v_{t}^{-i}}$ and $d_{t}^{-i}\triangleq\frac{v_{t}^{-i}}{v_{t}^{i}+v_{t}^{-i}}$ as the normalized values of $v^i_t$ and $v^{-i}_t$, respectively. To make better decisions, the historical observation information for agent $i$ up to slot $t$ is represented as $\tau^i_t\triangleq\{o_{t-M+1}^{i},...,o_{t-1}^{i},o_{t}^{i}\}$, where $M$ is the length of the observation history. The global state at time $t$ is $s_{t}\triangleq[\boldsymbol{a}_{t-1}, \boldsymbol{D}_{t-1}]$, where $\boldsymbol{a}_{t-1}$ is the joint action of all agents , and $\boldsymbol{D}_t\triangleq[D_{t}^{1},D_{t}^{2},...,D_{t}^{N}]$ with $D_{t}^{i}\triangleq\frac{v_{t}^{i}}{\sum_{i=1}^{N}v_{t}^{i}}$.

\textbf{Reward:}
Recall that our goal is to balance fairness among STAs and maximize the network throughput. Thus, we propose to encourage the agent with the largest delay to transmit. Specifically, the reward at time slot $t$ is represented by 
\begin{equation}
\begin{aligned}
& r_{t} =
\begin{cases}1, &\text{if\ agent\ $i=\mathop{\arg\max}\limits_{i}\boldsymbol{D}_t$\ transmits successfully,}\\[0.1ex] 
0, &\text{if\ no\ transmission,}\\[0.1ex]
-1, &\text{otherwise.}
\end{cases}
\end{aligned}
\end{equation}

\section{QPMIX Algorithm and Protocol}
This section introduces the proposed QPMIX algorithm. We begin by presenting the design of the agent networks, followed by an explanation of the loss function utilized to train these networks. Finally, similar to the valued-based QMIX algorithm \cite{rashid2020monotonic}, we propose a novel MARL algorithmic framework, named QPMIX, to account for its valued-based and policy-based heterogeneity, which operates a CTDE paradigm.

\subsection{QPMIX Neural Network Architecture}
\begin{figure*}[htbp]
\begin{minipage}[t]{0.5\textwidth}
\subfigure[Agent network]{
\includegraphics[width=3.2in ]{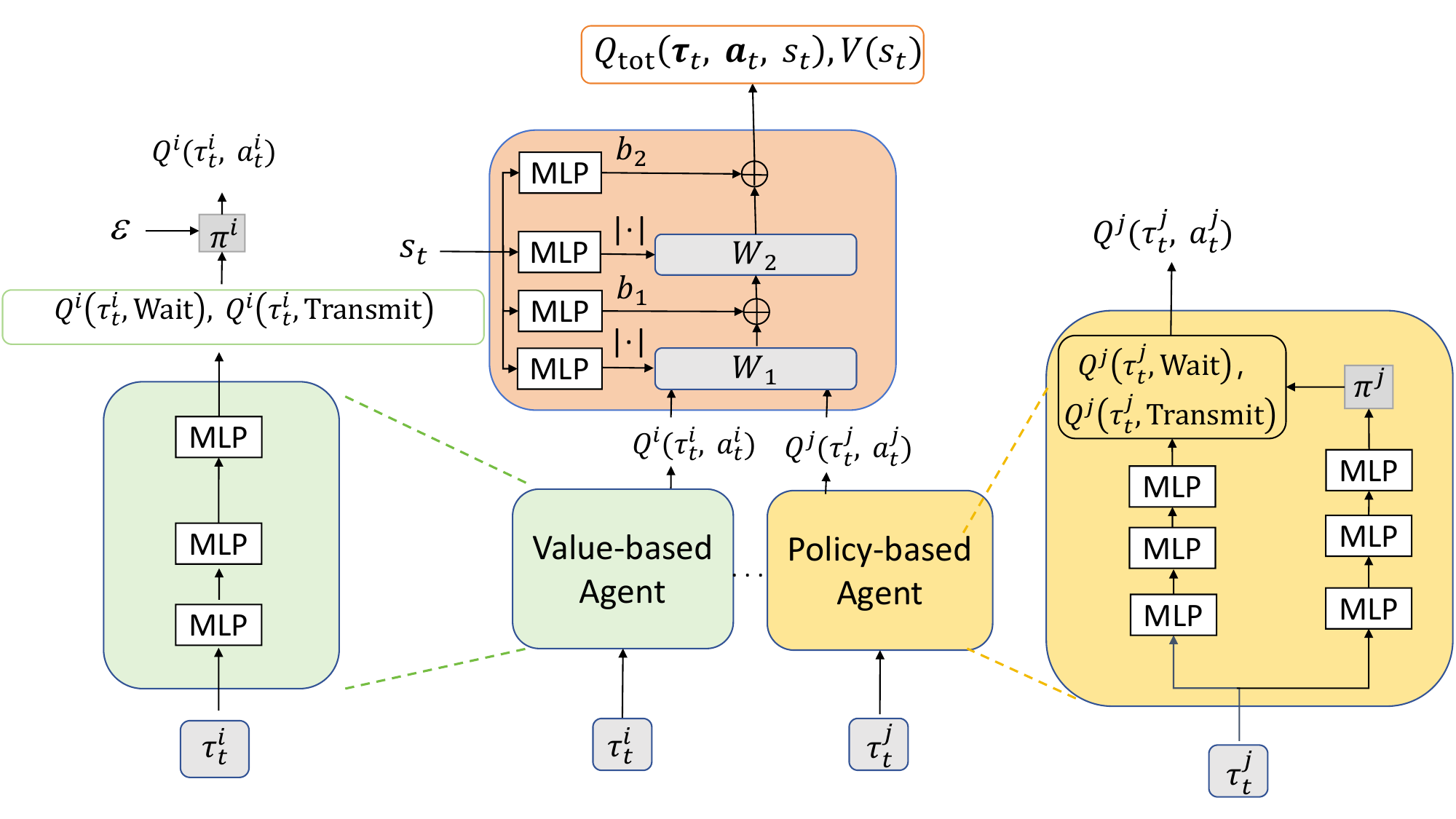}
\label{QPMIXFramework}
}
\end{minipage}
\begin{minipage}[t]{0.5\textwidth}
\subfigure[Training framework]{\includegraphics[width=3.2in]{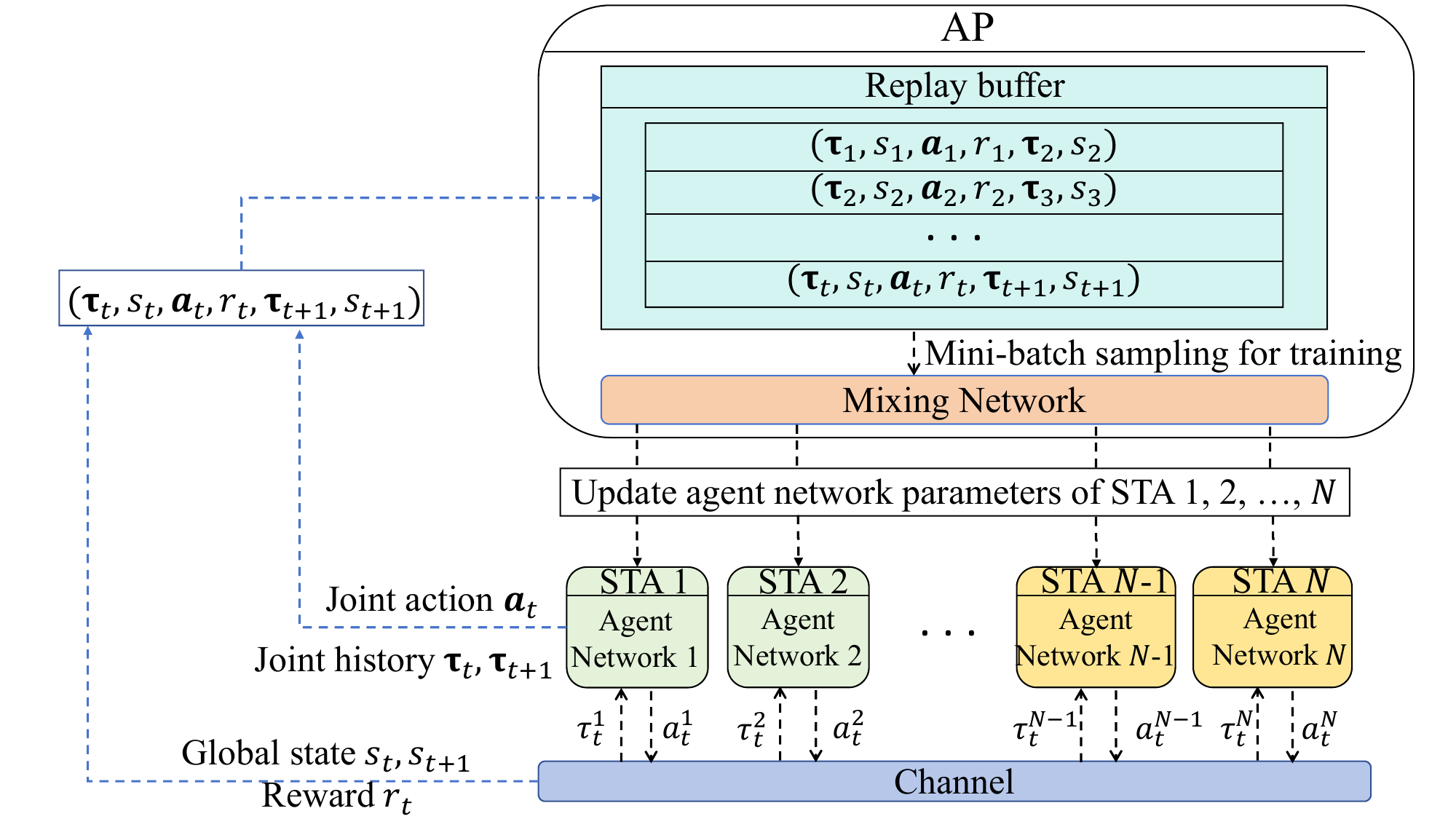}
\label{TrainingFramework}
}
\end{minipage}
\caption{QPMIX framework: (a) Agent network: The left and right block show the architecture of the value-based and policy-based networks, respectively. The middle block corresponds to hypernetworks that produce the weights and biases for the mixing network layers. 
(b) Training framework: The centralized training occurs at the AP based on the experiences reported by each STA. During the training phase, the AP distributes the gradients back to each STAs. During the decentralized execution phase, each STA independently decides whether to access the channel. 
}
\label{QPMIX}
\end{figure*}
The DQN \cite{mnih2015human} and PPO \cite{schulman2017proximal} are widely recognized value-based and policy-based RL algorithms, respectively.
Therefore, the value-based agent utilizes the DQN algorithm while the policy-based agent employs the PPO algorithm, as an example. 
As illustrated in Fig.~\ref{QPMIX}\subref{QPMIXFramework}, the QPMIX architecture consists of three components: the DQN agent network, the PPO agent network, and the mixing network.

\subsubsection{DQN and PPO Network}
The left and right blocks in Fig.~\ref{QPMIX}\subref{QPMIXFramework} represent the neural networks of the DQN and PPO agents, respectively. DQN agent $i$ utilizes its local historical observation $\tau^i_t$ as an input to a network comprising three multilayer perceptron (MLP) layers. Then the network outputs the Q-values for two actions, $Q^i(\tau^i_t, \text{Wait})$ and $Q^i(\tau^i_t, \text{Transmit})$. Based on these Q-values, the DQN agent utilizes an $\epsilon\text{-greedy}$ policy to select an action $a^i_t$, and inputs the corresponding Q-value $Q^i(\tau^i_t, a^i_t)$ into the mixing network. Similarly, PPO agent $j$ takes $\tau^j_t$ into both the critic and actor networks, shown on the left and right sides of the PPO network, respectively. The critic network outputs the Q-values for the two actions, while the actor produces the policy $\pi^j$, representing the probability distribution over the two actions. The agent then samples an action according to its policy $\pi^j$, and then inputs Q-value $Q^j(\tau^j_t, a^j_t)$ into the mixing network.

\subsubsection{Mixing Network}
The mixing network utilizes the global state $s_t$ as input to generate weights $W_1, W_2$ and biases $b_1, b_2$ to perform a weighted sum over the individual Q-values $Q^i(\tau^i_t, a^i_t), \forall i \in\mathcal{N}$, thereby producing the total Q-value $Q_\text{tot}(\boldsymbol{\tau}_t,\boldsymbol{a}_t,s_t)$, where $\boldsymbol{\tau}_t$ is joint historical observations, defined as $\boldsymbol{\tau}_t=[\tau^1_t,...,\tau^N_t]$. 
We modify the original mixing network within the QMIX algorithm by incorporating an additional output state value, represented as $V(s_t)$. 
The absolute activation function is used when generating weights $W_1$ and $W_2$ to satisfy the IGM properties \cite{rashid2020monotonic}, i.e., $\frac{\partial{Q}_\text{tot}}{\partial Q^i}\geq0, \forall i$.  Therefore, the mixing network can be actually regarded as a two-layer MLP.

\subsection{Loss Function}
The parameters of the individual Q-networks of all agents and the mixing network are collectively denoted by $\boldsymbol{\theta}_\text{tot}$. 
We parameterize the actor network of PPO agent $j$, actor networks of all PPO agents, and the global state value $V(s_t)$ by $\boldsymbol{\theta}^j_\text{a}$, $\boldsymbol{\theta}_\text{a}$, and $\boldsymbol{\theta}_\text{V}$, respectively.
The loss function of the QPMIX algorithm can be denoted by
\begin{equation}
\mathcal{L}_\text{QPMIX}=\mathcal{L}_\text{Qtot} + \mathcal{L}_\text{V} + \mathcal{L}_\text{actor},
\label{QPMIXloss}
\end{equation}
where $\mathcal{L}_\text{Qtot}$ is utilized to update the mixing network and the individual Q-networks of all agents, defined as \cite{rashid2020monotonic}
\begin{equation}
\label{QMIX LOSS}
\mathcal{L}_\text{Qtot}(\boldsymbol{\theta}_\text{tot})=\sum_{bs}\left[y_\text{tot}-Q_\text{tot}(\boldsymbol{\tau}_t,\boldsymbol{a}_t,s_t;\boldsymbol{\theta}_\text{tot})\right]^2,
\end{equation}
with $y_\text{tot}=r_t + \gamma\max_{\boldsymbol{a}^{\prime}} Q_\text{tot}(\boldsymbol{\tau}_{t+1},\boldsymbol{a}^{\prime},s_{t+1};\boldsymbol{\theta}^-_\text{tot})$. Here, $\boldsymbol{\theta}^-_\text{tot}$ is the target network parameters, which is fixed for a couple of updates and duplicated periodically from $\boldsymbol{\theta}_\text{tot}$ to stabilize training, and $bs$ is the batch size.
Moreover, 
$\mathcal{L}_\text{actor}$ is utilized to update the actor network, defined as 
\begin{equation}
\begin{aligned}
&\mathcal{L}_\text{actor}(\boldsymbol{\theta}_\text{a})=-\sum_{bs}\sum_{j}\min\left(\frac{\pi^j(a^j_t|\tau^j_t; \boldsymbol{\theta}^j_\text{a})}{\pi^j(a^j_t|\tau^j_t;{\boldsymbol{\theta}^j_\text{a,old}})}A(s_t; \boldsymbol{\theta}_\text{V}),\right.\\ 
&\left.\text{clip}\Big(\frac{\pi^j(a^j_t|\tau^j_t; \boldsymbol{\theta}^j_\text{a})}{\pi^j(a^j_t|\tau^j_t;{\boldsymbol{\theta}^j_\text{a,old}})},1-\delta,1+\delta\Big)A(s_t;\boldsymbol{\theta}_\text{V})\right),
\label{PPOloss}
\end{aligned}
\end{equation}
where $\text{clip}(\cdot, \cdot, \cdot)$ is the clipping function that clips the first input into the range determined by the second and third inputs, i.e., $[1-\delta, 1+\delta]$, with $\delta$ as the clipping factor, usually set to 0.2, and $\boldsymbol{\theta}_\text{a,old}$ is copied from the actor parameter before the update, and set to the new $\boldsymbol{\theta}_\text{a}$ after the update \cite{schulman2017proximal}.
Furthermore, $\mathcal{L}_\text{V}$ is defined as
\begin{equation}
\label{V_LOSS}
\mathcal{L}_\text{V}(\boldsymbol{\theta}_\text{V})=\sum_{bs}\big(r_t+\gamma V(s_{t+1};\boldsymbol{\theta}^-_\text{V})-V(s_t;\boldsymbol{\theta}_\text{V})\big)^2,
\end{equation}
where $\boldsymbol{\theta}^-_\text{V}$ is the parameters of the target global state value function, and $\boldsymbol{\theta}_\text{V}$ is the global state value function that will be used to estimate the advantage function $A(s_t; \boldsymbol{\theta}_\text{V})$ in (\ref{PPOloss}), which is defined as
\begin{equation}
\label{advantage function}
A(s_t; \boldsymbol{\theta}_\text{V}) =\delta_t+\gamma\lambda\delta_{t+1}+\cdots+(\gamma\lambda)^{T-t+1}\delta_{T-1},
\end{equation}
where $\lambda$ is a parameter to trade off between the bias and variance of the advantage function, and $\delta_t$ is the temporal difference (TD) error given by \cite{schulman2017proximal}
\begin{equation}
\delta_t = r_t+\gamma V(s_{t+1};\boldsymbol{\theta}^-_\text{V})-V(s_t;\boldsymbol{\theta}_\text{V}).
\end{equation}


\subsection{Learning Algorithm}
Based on the QPMIX architecture introduced above, we propose a training framework for addressing the DCA problem as illustrated in Fig.~\ref{QPMIX}\subref{TrainingFramework}. 
During the centralized training phase, the AP computes the reward $r_t$ according to the reward function after each agent selects an action based on its own historical observation information $\tau^i_t$. Then, each agent records its own historical observation $\tau^i_t$ and action $a^i_t$, and uploads these recorded values $\tau^i_t, a^i_t$ to the AP every $N_c$ time steps, where $N_c$ represents the networks update interval. 
The AP concatenates these into vectors $\boldsymbol{\tau}_t$ and $\boldsymbol{a}_t$. As the AP can determine the elapsed time since the last successful packet transmission from each STA, it computes $D^i_t, \forall i \in \mathcal{N}$, thereby obtaining global state $s_t$. Consequently, the AP maintains a replay buffer that consists of $(\boldsymbol{\tau}_t, s_t, \boldsymbol{a}_t, r_t, \boldsymbol{\tau}_{t+1}, s_{t+1})$, and samples experiences from the replay buffer to compute gradients of each agent network and mixing network parameters by computing the loss function in (\ref{QPMIXloss}). As a result, the AP updates the mixing network model and offloads the gradients to each STA to update their the models. The AP also maintains a network update counter $C_t$ to track the number of updates. Whenever $C_t\ \text{mod}\ N_t == 0$, the target network parameters $\boldsymbol{\theta}^-_\text{tot}$ and $\boldsymbol{\theta}^-_\text{V}$ are copied as $\boldsymbol{\theta}_\text{tot}$ and $\boldsymbol{\theta}_\text{V}$, respectively, where $N_t$ is the target networks update interval. The data exchange between the AP and STAs described above occurs only during the centralized training phase. During the decentralized execution phase, each STA independently decides when to access the channel based on its own observation history.
The training procedure of the proposed QPMIX algorithm is summarized in Algorithm~\ref{QPMIXalgorithm}.


\begin{algorithm}[!t]
\caption{QPMIX Algorithm} 
\label{QPMIXalgorithm}
\SetKwData{Left}{left}\SetKwData{This}{this}\SetKwData{Up}{up}
 \SetKwFunction{Union}{Union}\SetKwFunction{FindCompress}{FindCompress}
 \SetKwInOut{Input}{Initialization}\SetKwInOut{Output}{output}
\Input{$\epsilon, N_c, N_t, T, t =0, C_t=0$, $s=s_0$, $\boldsymbol{\theta}^-_\text{tot}=\boldsymbol{\theta}_\text{tot}$, $\boldsymbol{\theta}^-_\text{V}=\boldsymbol{\theta}_\text{V}$, $\boldsymbol{\theta}_\text{a,old}^j=\boldsymbol{\theta}_\text{a}^j, \forall j \in \mathcal{N}_2$, $\tau^i_0=\tau_0, \forall i \in \mathcal{N}$. }

\While{$t<T$}{
\For{\textup{Agent} $i=1,2...,N$}{
    \eIf{\textup{Channel is idle}}{
        Agent $i$ obtains $\tau^i_t$ from $\tau^i_{t-1}$, $a^i_{t-1}$.\\
        \eIf{\textup{Agent is value-based}}
        {Value-based agent chooses an action by $\epsilon\text{-}$greedy policy.\\}
        {Policy-based agent chooses an action by policy $\pi^i$.}
    }{
        \eIf{\textup{$a^i_{t-1}$ is the Transmit action}}{
            \eIf{\textup{Agent $i$ does not finish transmission}}{
                Agent $i$ transmits.
            }{
                Agent $i$ waits.}
        }{
            Agent $i$ waits.
        }
    }
}
Update the channel state, reward $r_t$, and the global state $s_{t+1}$ based on joint action $\boldsymbol{a}_t$.\\
Store $(\boldsymbol{\tau}_t, s_t, \boldsymbol{a}_t, r_t, \boldsymbol{\tau}_{t+1}, s_{t+1})$ to replay buffer.\\
\If{$t$ \textup{mod} $N_c$ == 0}{
Sample experiences from replay buffer and compute (\ref{QPMIXloss}).\\
Update $\boldsymbol{\theta}_\text{tot}$, $\boldsymbol{\theta}_\text{V}$, and $\boldsymbol{\theta}^j_\text{a},\ \forall j \in \mathcal{N}_2$ by gradient descent.\\
$C_t \leftarrow C_t + 1, \boldsymbol{\tau}_t \leftarrow \boldsymbol{\tau}_{t+1}, s_t \leftarrow s_{t+1}$, $\boldsymbol{\theta}^j_\text{a,old} \leftarrow \boldsymbol{\theta}^j_\text{a},\ \forall j \in \mathcal{N}_2$.
}
\If{$C_t$ \textup{mod} $N_t$ == 0}{
$\boldsymbol{\theta}^-_\text{tot} \leftarrow \boldsymbol{\theta}_\text{tot}, \boldsymbol{\theta}^-_\text{V} \leftarrow \boldsymbol{\theta}_\text{V}$.
}
}
\end{algorithm}

\section{Heterogeneous MARL Convergence Analysis}
In this section, we provide a theoretical analysis of the proposed QPMIX algorithm. Specifically, we prove the convergence of the QPMIX algorithm when using the linear value function approximation. 
We focus on the cooperative task setting, where agents are trained to maximize the joint expected return. For simplicity, we assume that the global states are fully observable and the rewards are identical for all agents. 
The policy-based agent employs an actor-critic algorithm, so both the value-based agents and the policy-based agents need to estimate the joint-action value function $Q(\cdot,\cdot; \omega)\colon\mathcal{S}\times\mathcal{A}\to\mathbb{R}$. 
The Q-function of value-based agent $l$ is approximated by a linear function $Q^l{\left(s,\boldsymbol{a};\omega^l\right)},\omega^l\in\mathbb{R}^K$, and $l\in\mathcal{N}_1$, where $s \in \mathcal{S},  \boldsymbol{a} \in \mathcal{A}$ are the global state and joint action, respectively. Moreover, for a policy-based agent, both its critic and actor networks are represented by linear functions, denoted respectively as $Q^j(s,\boldsymbol{a};\omega^j)$, with $\omega^j\in\mathbb{R}^K$, and $\pi_p^j(s,a^j;\theta^j)$ with $\theta^j\in\mathbb{R}^M$, where $j\in\mathcal{N}_2$. 
Before proceeding to demonstrate the convergence of heterogeneous MARL, we make the following assumptions. 
\begin{assumption} \label{assump1}
The Q-function is defined as a linear function, $Q(s,\boldsymbol{a};\omega)=\omega^\top\phi(s,\boldsymbol{a})$, where $\phi(s,\boldsymbol{a})\in\mathbb{R}^{K}$ represents the feature vector of the $(s,\boldsymbol{a})$ pair and is uniformly bounded for any $s\in\mathcal{S},\boldsymbol{a}\in\mathcal{A}$. Moreover, feature matrix $\Phi\in\mathbb{R}^{|\mathcal{S}||\mathcal{A}|\times K}$ has full column rank and reward function $R(s,\boldsymbol{a})$ is uniformly bounded for any $s\in\mathcal{S},\boldsymbol{a}\in\mathcal{A}$, where $|\mathcal{S}|,|\mathcal{A}|$ denote the cardinality of the set $\mathcal{S}$ and $\mathcal{A}$, respectively. 
\end{assumption}

\begin{assumption} \label{assump2}
The joint policy of all agents is given by $\pi_{\Theta}\left(s,\boldsymbol{a}\right)=\prod_{l\in\mathcal{N}_1}\pi_{v}^{l}\left(s,a^{l};\omega^{l}\right)\prod_{j\in\mathcal{N}_2}\pi_{p}^{j}\left(s,a^{j};\theta^{j}\right)$, where $\pi_{p}^{j}(s,a^{j};\theta^{j})$ is continuously differentiable with respect to the parameter $\theta^j$, and $\pi^l_v(s, a^l;\omega^{l})$ is the $\epsilon\text{-}$greedy policy of value-based agent $l$ derived by its Q-function.
Suppose $P^{\Theta}$ is the Markov transition probability matrix derived from the policy $\Theta$. Then for any state $s_t,s_{t+1} \in S$ we have:
\begin{equation}P^\Theta(s_{t+1}|s_t)=\sum_{\boldsymbol{a}_t\in\mathcal{A}}\pi_\Theta(s_t,\boldsymbol{a}_t)P(s_{t+1}|s_t,\boldsymbol{a}_t),\end{equation}
where $P(s_{t+1}|s_t,\boldsymbol{a}_t)$ is the state transition probability of the MDP. Furthermore, we assume that Markov chain $\{s_t\}_{t\geq0}$ is irreducible and aperiodic under any joint policy $\pi_{\Theta}$, with its stationary distribution denoted by $d_{\Theta}(s)$ \cite{konda1999actor}. 
\end{assumption}

\begin{assumption} \label{assump3}
The step size of the Q-function update for the value-based agent and the policy-based agent, denoted by $\beta_{\omega,t}$, and the step size of the actor network update for the policy-based agent, denoted by $\beta_{\theta,t}$, satisfy the following conditions
\begin{small}
\begin{equation}
\begin{aligned}
&\sum_t\beta_{\omega,t}=\infty, \sum_t\beta_{\theta,t}=\infty,\sum_t\beta_{\omega,t}^2<\infty,\\[0.1ex]
&\sum_t\beta_{\theta,t}^2<\infty, \lim_{t\to\infty}\frac{\beta_{\theta,t}}{\beta_{\omega,t}}\to0,
\lim_{t\to\infty}\frac{\beta_{\omega,t+1}}{\beta_{\omega,t}}=1.
\end{aligned}
\end{equation}
\end{small}
\end{assumption}
\begin{remark}
Assumption \ref{assump3} is to ensure that the update of the Q-function occurs much faster than that of the actor, thus enabling us to utilize the two-time-scale stochastic approximation theorem \cite{borkar2000ode}.
\end{remark}

To analyze the progressive behavior of agents, the heterogeneous MARL algorithm is divided into two steps: the Q-function update step for both value-based and policy-based agents, and the actor update step for policy-based agents. In the Q-function update step, each agent performs an update based on TD learning to estimate $Q(\cdot,\cdot;\omega^i)$. For simplicity, we use $Q^i_t(\omega^i_t)$ to represent $Q^i(s_t,\boldsymbol{a}_t;\omega_t^i)$. The iteration proceeds as follows:
\begin{equation} 
\begin{cases}
\label{critic step}
\omega_{t+1}^i=\omega_t^i+\beta_{\omega,t}\cdot\hat{\delta}_t^i\cdot\nabla_\omega Q^i_t(\omega_t^i),\\
\hat{\delta}_{t}^{i}=r_{t}+N^{-1}\sum_{k}^{N}c_{t}(i,k)
\left[\gamma Q^{k}_{t+1}(\omega_{t+1}^{k})-Q^{k}_t(\omega_{t}^{k})\right],
\end{cases}
\end{equation}
where $\beta_{\omega,t}>0$ is the step size, and $c_t(i,k)$ is the weight of the TD-error information transmitted from agent $k$ to the agent $i$ at time $t$. The weight matrix $C_t=[c_t(i,k)]\in\mathbb{R}^{N\times N}$ is assumed to be column stochastic, i.e., $\mathbf{1}^\top\mathrm{C}_t=\mathbf{1}$, where $\mathbf{1} \in \mathbb{R}^{N}$ is a vector with all elements equal to one. In fact, when we take the partial derivative of $\mathcal{L}_\text{Qtot}$ with respect to the individual Q-function parameter of agent $i$ in (\ref{QMIX LOSS}), $\frac{\partial{\mathcal{L}}_\text{Qtot}}{\partial Q^i}$ contains a weighted sum over the Q-values of all agents. 
Hence, minimizing $\mathcal{L}_\text{Qtot}$ aligns with the operation in (\ref{critic step}). 

As for the actor update step, each policy-based agent $j$ improves its policy using policy gradient
\begin{equation}
\begin{cases}
\label{equation 2}
\theta_{t+1}^{j}=\theta_{t}^{j}+\beta_{\theta,t}\cdot A_{t}^{j}\cdot\nabla_{\theta^{j}}\log\pi^{j}(\theta_{t}^{j}),\\A_{t}^{j}=Q_t^{j}(\omega_{t}^{j})-\sum_{a^{j}\in\mathcal{A}}\pi^{j}_t(s_t, a_t^{j};\theta_{t}^{j})\cdot Q_{t}^{j}(s_{t},a_{t}^{j},a_{t}^{-j};\omega_{t}^{j}),\\
\end{cases}
\end{equation}
where $\beta_{\theta,t}>0$ is the step size. Please note that the process of updating the PPO model parameters using the loss function in (\ref{PPOloss}), along with the advantage function defined in (\ref{advantage function}), represents a variation of the policy update method and the advantage function used in the original policy gradient algorithm in (\ref{equation 2}). 


Next we utilize the two-time-scale stochastic approximation technique \cite{borkar2009stochastic} to analyze the convergence of $\omega^i_t$ in (\ref{critic step}) and $\theta^j_t$ in (\ref{equation 2}). Assume that the policy of the value-based agent is periodically derived from its Q-function. Specifically, the Q-function parameters in (\ref{critic step}) are periodically duplicated, and the policy is determined using $\epsilon\text{-greedy}$ based on this copied Q-function. Consequently, the policy update of the value-based agent is slower than the Q-function update.  Furthermore, given that the policy updates of policy-based agents are also considerably slower compared to the updates of the Q-function under Assumption \ref{assump3}, the joint policy $\pi_\Theta$ update can be considered static relative to the Q-function. As a result, the general idea of the proof is to establish the convergence of $\omega^i_t$ at a faster time scale, during which $\theta^j_t$ can be regarded as fixed, and subsequently demonstrate the gradual convergence of $\theta^j_t$, inspired by \cite{zhang2018fully}. This ensures the eventual convergence of both Q-functions and actors.


\begin{theorem}
\label{omegatheorem}
Under Assumptions \ref{assump1}-\ref{assump3}, we have $\lim_{t\to\infty}\omega^i_t=\omega_\Theta, \forall i\in \mathcal{N}$ almost surely (a.s., i.e., convergence with probability $1$), for any given joint policy parameters $\Theta$. $\omega_\Theta$ is the unique solution to
\begin{equation}
\label{Theorem1 solution}
\Phi^\top\mathrm{D}_\Theta\left[R+\left(\gamma P^{\pi}-I\right)\Phi\omega_\Theta\right]=0, 
\end{equation}
where $\mathrm{D}_\Theta=\textup{diag}[d_\Theta(s)\pi_\Theta(s,\boldsymbol{a}), s\in\mathcal{S},  \boldsymbol{a}\in\mathcal{A}]\in\mathbb{R}^{|\mathcal{S}||\mathcal{A}|\times|\mathcal{S}||\mathcal{A}|}$ is a diagonal matrix, and $d_\Theta(s)\pi_\Theta(s,\boldsymbol{a})$ is an element on the diagonal. $I\in\mathbb{R}^{|\mathcal{S}||\mathcal{A}|\times|\mathcal{S}||\mathcal{A}|}$ is the identity matrix, $R$ is short for reward function $[R(s,\boldsymbol{a}) , s\in\mathcal{S}, \boldsymbol{a}\in\mathcal{A}]\in\mathbb{R}^{|\mathcal{S}||\mathcal{A}|}$, and $P^{\pi}$ is short for $P^{\pi}(s_{t+1}, \boldsymbol{a}_{t+1} |s_t,\boldsymbol{a}_t) = P(s_{t+1}|s_t,\boldsymbol{a}_t) \pi_{\Theta}(s_{t+1}, \boldsymbol{a}_{t+1})$.
\end{theorem}
\begin{proof}
See Appendix~\ref{Appendix A}.
\end{proof}

\begin{remark}
Theorem \ref{omegatheorem} ensures that the parameters of Q-function $\omega^i$ for each agent, including both value-based and policy-based agents, will converge to the same point over time a.s. with the sequence ${\omega^i_{t}}$ generated from (\ref{critic step}), for a given joint policy.
\end{remark}

\begin{theorem}
\label{thetatheorem}
Under Assumptions \ref{assump1}-\ref{assump3}, $\theta^j_t$ converges a.s. to a point in the set of the asymptotically stable equilibria of 
\begin{equation}\dot{\theta}^j=\lim_{0<\eta\to0}\Big\{\Gamma^j\big[\theta^j+\eta\mathbb{E}_{s_t\sim d_\Theta,a_t\sim\pi_\theta}\left(A_{t,\theta}^j\cdot\psi_t^j\right)\big]-\theta^j\Big\}\Big/\eta,
\label{Theorem2}
\end{equation}
where $\Gamma^j(\cdot)$ is a projection function, $\psi^j_t=\nabla_{\theta^{j}}\log\pi^{j}(\theta_{t}^{j})$, and $A_{t,\theta}^j$ is the advantage function when $\omega^j_t$ converges to $\omega_\Theta$, i.e., $A_{t, \theta}^{j}=Q^{j}(\omega_\Theta)-\sum_{a^{j}\in\mathcal{A}}\pi^{j}_t(s_t, a_t^{j};\theta_{t}^{j})\cdot Q_{t}^{j}(s_{t},a_{t}^{j},a_{t}^{-j};\omega_\Theta)$

\end{theorem}
\begin{proof}
See Appendix~\ref{Appendix C}.
\end{proof}

\begin{remark}
Theorem \ref{thetatheorem} states that as the Q-function of each agent converges, the advantage function computed by the policy-based agent also converges, which allows $\theta_t^j$ for $j\in \mathcal{N}_2$ to gradually approach a stationary point of the ordinary differential equation in (\ref{Theorem2}). Since all agents receive the same reward and each Q-function estimates the value for the global state and joint actions, both value-based and policy-based agents update their policies in a manner that maximizes the Q-function. Combining Theorem \ref{omegatheorem} and Theorem \ref{thetatheorem}, we establish the convergence of the QPMIX algorithm.
\end{remark}

\section{Performance Evaluation}
This section provides simulation results of the proposed QPMIX algorithm in different scenarios, including saturated traffic, unsaturated traffic, and delay-sensitive traffic scenarios. We use the throughput, mean delay, delay jitter, collision rate, and Jain fairness index (JFI) metrics to evaluate the performance of the proposed QPMIX algorithm as characterized as follows. 
\begin{itemize}
\item Throughput: The ratio of successful transmission slots to total time slots.
\item Mean delay: The average delay of all successfully transmitted packets, measured as the number of time slots from packet generation to successful transmission.
\item Delay jitter: The variance in the delay of all successfully transmitted packets.
\item Collision rate: The proportion of packets that collide relative to all packets sent.
\item JFI: This is used to measure the fairness among STAs, defined as 
\begin{equation}
\text{JFI}=\frac{\left(\sum_{i=1}^Ne_i\right)^2}{N\sum_{i=1}^N\left(e_i\right)^2},
\end{equation}
where $N$ is the number of STAs and $e_i$ is the throughput of STA $i$. The range of JFI is $[\frac{1}{N},1]$, where $\frac{1}{N}$ means that one STA exclusively occupies the channel and 1 means all STAs equally share the channel.

\end{itemize}


\subsection{Simulation Setup}
\subsubsection{Simulation Scenario}
We consider a channel access scenario for a Wi-Fi BSS, as shown in Fig.~\ref{BSS}. 
Each STA is equipped with a packet buffer of length 10, operating on a first-in, first-out basis.
Both saturated traffic and unsaturated traffic scenarios are considered, where packets arrive following Poisson distribution. The packet arrival rates are $\lambda_1=2000\ \text{packets}/ \SI{}{\second}$ and $\lambda_2=200\ \text{packets} / \SI{}{\second}$, corresponding to high and low traffic intensities in the saturated and unsaturated traffic scenarios, respectively. 
In addition, we also consider a delay-sensitive traffic scenario, where packets are generated periodically with a generation period of $T_\text{v}=\SI{20}{\milli\second}$, i.e., voice over internet protocol (VoIP) traffic.
\begin{table}[!t]
\caption{Simulation Parameters of EDCA}
\label{EDCA parameters}
\centering
\begin{tabular}{|c|c|c|c|}
     \hline 
     \textbf{Parameters} & \textbf{AC\_VO} & \textbf{AV\_VI} & \textbf{AC\_BE} \\
     \hline 
     Time Slot ($\upmu$s) & \multicolumn{3}{|c|}{9}\\
     \hline
     Short Interframe Space ($\upmu$s) & \multicolumn{3}{|c|}{18} \\
     \hline
     Distributed Interframe Space ($\upmu$s) & \multicolumn{3}{|c|}{36} \\
     \hline
     Packet Length ($\upmu$s) & \multicolumn{3}{|c|}{1080} \\
     \hline
     $\text{CW}_\text{min}$ & 7&15&31 \\ 
     \hline
     $\text{CW}_\text{max}$ & 15& 31& 1023 \\
     \hline
\end{tabular}
\end{table}

\begin{table}[!t]
\caption{Hyper-Parameters of QPMIX}
\label{Hyper-Parameters of QPMIX}
\centering
    \begin{tabular}{|c|c|}
        \hline
        \textbf{Parameters} & \textbf{value} \\
        \hline
        Agent update interval $N_c$ & 10 \\
        \hline
        Target network update interval $N_t$ & 1000 \\
        \hline
        Replay buffer size & 500 \\
        \hline
        DQN batch size & 32 \\
        \hline
        Discount factor $\gamma$ & 0.5 \\
        \hline
        Range of $\epsilon$ & 1 to 0.01 \\
        \hline
        Decay rate of $\epsilon$ & 0.998 \\
        \hline 
        DQN and mixing network learning rate & $5 \times 10^{-4}$\\
        \hline 
        PPO learning rate & $1 \times 10^{-5}$ \\
        \hline
        Neurons of DQN agent Q-network & 250, 120, 120\\
        \hline
        Neurons of PPO agent critic & 250, 120, 120\\
        \hline
        Neurons of PPO agent actor & 250, 120, 120\\
        \hline
        Neurons of mixing network & 16 \\
        \hline
    \end{tabular}
\end{table}

\subsubsection{Baselines}
We compare the proposed QPMIX algorithm with two homogeneous MARL algorithms, i.e., the QLBT algorithm \cite{guo2022multi} and MFMAPPO algorithm \cite{han2024multi}, as well as with the conventional Wi-Fi enhanced distributed channel access (EDCA) methods. 
EDCA includes different access categories (ACs), such as AC voice (AC\_VO), AC video (AC\_VI) and AC best effort (AC\_BE), each with distinct channel access parameters. The detailed parameters are shown in Table~\ref{EDCA parameters}.

\subsubsection{Parameters of QPMIX}
For both DQN and PPO agents, the rectified linear unit (ReLU) is used as the activation function for each layer of neurons, except for the last layer. The RMSProp optimizer is employed to update the network parameters. Note that the learning rate for DQN agents is set to $5 \times 10^{-4}$, significantly higher than that of PPO agents, $1 \times 10^{-5}$. This setting is to ensure maximum compliance with Assumption~\ref{assump3}.
The hyper-parameters of the QPMIX algorithm are shown in Table~\ref{Hyper-Parameters of QPMIX}.

\stepcounter{scenario}
\subsection{Scenario \Roman{scenario}: Saturated Traffic}
Fig.~\ref{Performance comparison} shows the throughput, mean delay, delay jitter and collision rate performance of the proposed QPMIX algorithm as compared against the baselines. From Fig.~\ref{Performance comparison}\subref{throughpt}, as the number of STAs increases, the throughputs drop for all algorithms due to the increased collision rate when more STAs attempt to access the channel, as shown in Fig.~\ref{Performance comparison}\subref{collision_rate}. While the proposed algorithm achieves slightly higher throughput than that of QLBT and MFMAPPO algorithms, our primary purpose is to use them as benchmarks to gauge the performance of the proposed QPMIX algorithm for heterogeneous MARL scenarios, rather than to demonstrate and emphasize the performance gains.
We also observe from Fig.~\ref{Performance comparison}\subref{mean_delay} and \ref{Performance comparison}\subref{delay_jitter} that the proposed QPMIX algorithm achieves lower mean delay and delay jitter, which becomes increasingly apparent as the number of STAs increases.
Based on Fig.~\ref{Performance comparison}\subref{collision_rate}, when the number of the STAs increases, the collision rate of the QPMIX algorithm is significantly lower than that of the EDCA ACs. Specifically, the collision rate of the QPMIX algorithm is reduced by 50\% compared to that of AC\_BE (the best in the EDCA categories) in the nine-STA case. 
Furthermore, Table~\ref{JFI} shows the JFI performance of different algorithms when the number of STA varies from two to nine. Compared to other algorithms, QPMIX achieves the highest fairness, whose JFI is very close to 1, corroborating the effectiveness of our proposed QPMIX algorithm in facilitating fair access of nodes to the channel.

\label{Saturated Traffic}
\begin{figure*}[!ht]
\begin{minipage}[t]{0.5\textwidth}
\subfigure[Throughput]{
\includegraphics[width=3.2 in ]{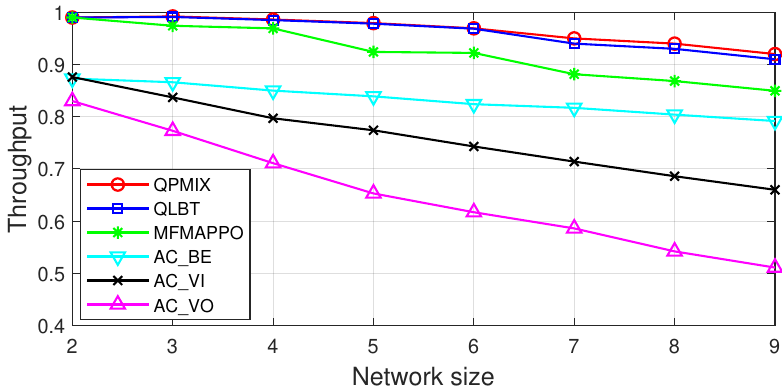}
\label{throughpt}
}
\end{minipage}
\begin{minipage}[t]{0.5\textwidth}
\subfigure[Mean delay]{\includegraphics[width=3.2 in]{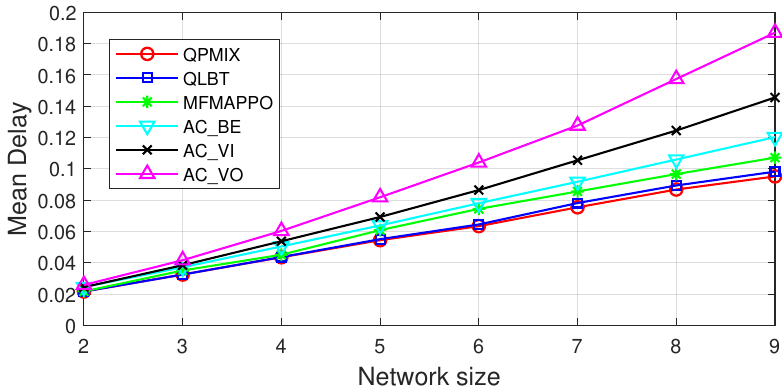}
\label{mean_delay}
}
\end{minipage}
\\
\begin{minipage}[t]{0.5\textwidth}
\subfigure[Delay jitter]{
\includegraphics[width=3.2 in ]{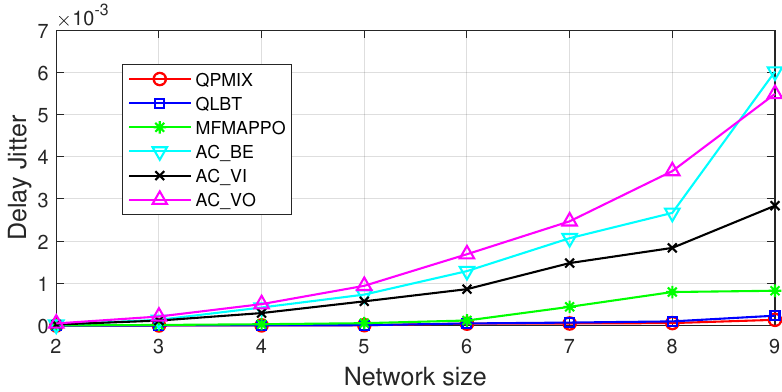}
\label{delay_jitter}
}
\end{minipage}
\begin{minipage}[t]{0.5\textwidth}
\subfigure[Collision rate]{\includegraphics[width=3.2 in]{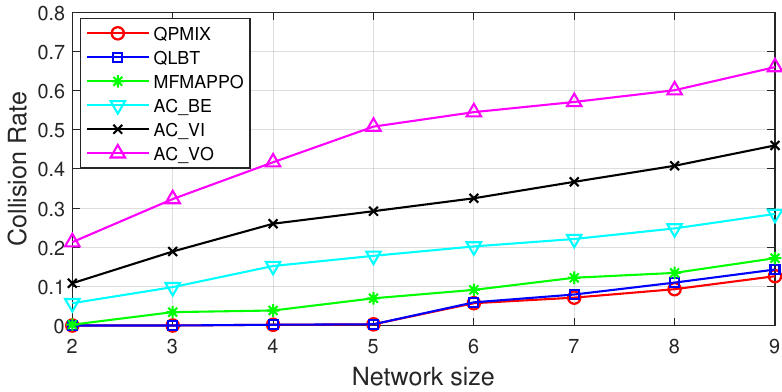}
\label{collision_rate}
}
\end{minipage}
\caption{Performance comparison under saturated Poisson traffic. For QPMIX, the even number in network size axis indicates the same number of the DQN and PPO STAs, while the odd number indicates one more DQN than PPO STAs.}
\label{Performance comparison}
\end{figure*}


\begin{table}[!t]
\scriptsize
\renewcommand{\arraystretch}{1}  
\setlength{\tabcolsep}{4pt}
\caption{JFI performance comparison among each method}
\label{JFI}
\centering
\begin{tabular}{|>{\centering\arraybackslash}c|c|c|c|c|c|c|c|c|}
\hline 
\multirow{2}{*}[-2pt]{\footnotesize\centering \textbf{Method}} & 
\multicolumn{8}{c|}{\footnotesize \textbf{JFI at different STA numbers}} \\
\cline{2-9}
& \footnotesize 2 & \footnotesize 3 & \footnotesize 4 & \footnotesize 5 & \footnotesize 6 & \footnotesize 7 & \footnotesize 8 & \footnotesize 9 \\
\hline
QPMIX & 0.999 & 0.999 & 0.999 & 0.999 & 0.998 & 0.997 & 0.995 & 0.994 \\ 
\hline
QLBT & 0.999 & 0.999 & 0.997 & 0.997 & 0.996 & 0.993 & 0.993 & 0.992 \\ 
\hline
MFMAPPO & 0.999 & 0.999 & 0.997 & 0.997 & 0.994 & 0.995 & 0.992 & 0.990 \\ 
\hline
AC\_BE & 0.999 & 0.992 & 0.992 & 0.977 & 0.972 & 0.974 & 0.970 & 0.956 \\
\hline
AC\_VI & 0.999 & 0.999 & 0.994 & 0.987 & 0.990 & 0.985 & 0.986 & 0.988 \\
\hline
AC\_VO & 0.999 & 0.998 & 0.985 & 0.996 & 0.993 & 0.983 & 0.948 & 0.949 \\
\hline
\end{tabular}
\end{table}



\begin{figure*}[!ht]
\begin{minipage}[t]{0.5\textwidth}
\subfigure[Unsaturated traffic scenario.]{
\includegraphics[width=3.2 in ]{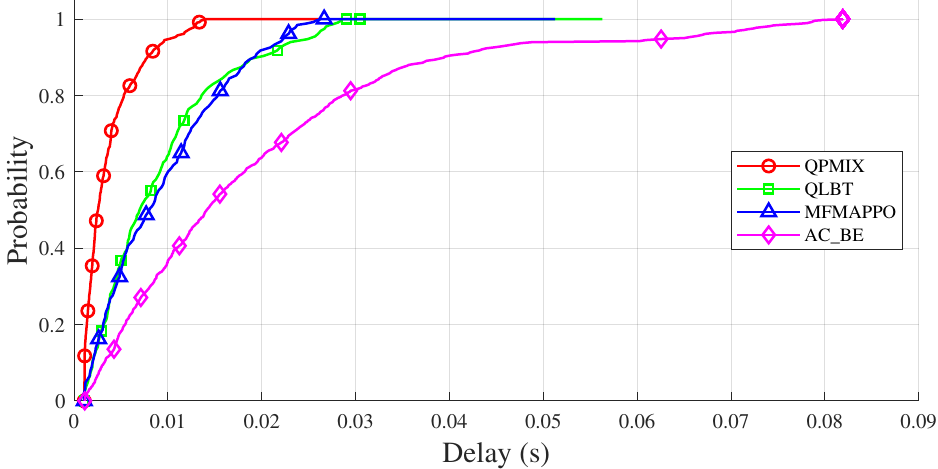}
\label{unsaturated cdf}
}
\end{minipage}
\begin{minipage}[t]{0.5\textwidth}
\subfigure[Delay-sensitive traffic scenario.]{\includegraphics[width=3.2 in]{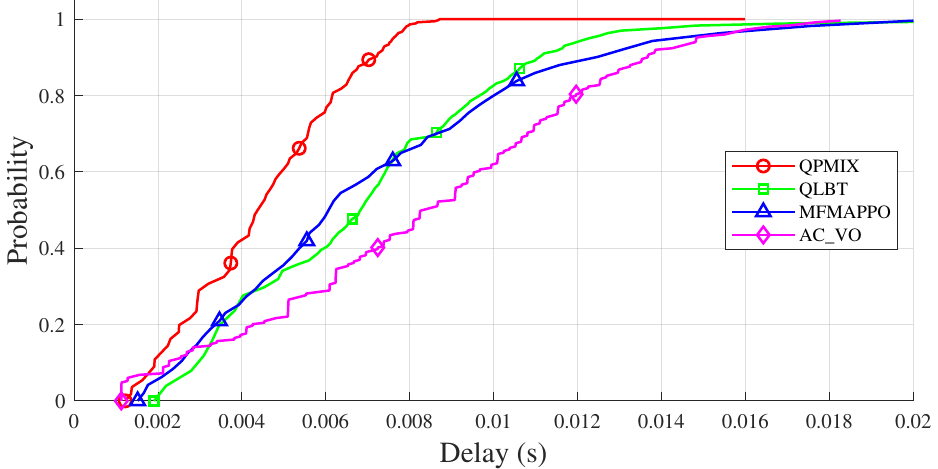}
\label{delay-senstive}
}
\end{minipage}
\caption{Delay CDF comparison among different algorithms under unsaturated and delay-sensitive traffic scenarios.}
\label{unsaturated}
\end{figure*}

\stepcounter{scenario}
\subsection{Scenario \Roman{scenario}: Unsaturated Traffic}
In this subsection, we evaluate the robustness of the proposed QPMIX by training in a saturated traffic scenario and testing in a unsaturated traffic scenario. 
We consider a network comprising four STAs, each with the same packet arrival following a Poisson distribution of $\lambda_2=200 \ \text{packets}/\SI{}{\second}$.
We plot in Fig.~\ref{unsaturated}\subref{unsaturated cdf} the average cumulative distribution function (CDF) of the delay experienced by the STAs when using QPMIX, QLBT, MFMAPPO and AC\_BE, where the average is taken over the four STAs. Note that AC\_BE is used since it achieves the best performance for Poisson traffic among different EDCA ACs as demonstrated in Section~\ref{Saturated Traffic}. 
As illustrated in Fig.~\ref{unsaturated}\subref{unsaturated cdf}, 
the delay upper bounds for QPMIX are only $0.015\ \SI{}{\second}$, which is significantly lower than the delay upper bounds of other baselines.
Therefore, we conclude that the QPMIX algorithm outperforms other baseline methods in terms of the delay performance, even when the traffic density in the testing phase differs from that of the training.

\stepcounter{scenario}
\subsection{Scenario \Rmnum{3}: Delay-Sensitive Traffic}
Fig.~\ref{unsaturated}\subref{delay-senstive} shows the average CDF of delay performance for VoIP traffic using QPMIX, QLBT, MFMAPPO and AC\_VO, with eight STAs in the network. Since VoIP is one of AC\_VO traffic, we use AC\_VO as the EDCA baseline.
The packet generation period is set to $\SI{20}{\milli\second}$ and the models are trained using the saturated Poisson traffic for robustness evaluation. 
It can be observed that the QPMIX algorithm has a lower delay upper bound and steeper CDF curve than other baselines, suggesting that the QPMIX algorithm outperforms other MARL algorithms and traditional random access methods in mean delay and delay jitter. 
The evaluation of both the unsaturated and delay-sensitive traffic scenarios demonstrates that the proposed QPMIX algorithm remains effective across various scenarios, even when the training and testing conditions are different, showing strong robustness.




\stepcounter{scenario}
\subsection{Scenario \Roman{scenario}: Different Mixings of Valued- and Policy-based Agents}
In Fig.~\ref{asymmetric_net} we evaluate the performance for different mixings of DQN and PPO agents to assess the adaptability of the proposed algorithm in heterogeneous MARL, where the numbers of DQN and PPO agents vary. We show the real-time throughput during the training phase for two distinct asymmetric agents configurations under saturated traffic: three DQN agents with one PPO agent, and one DQN agent with three PPO agents. From the figure, despite the disparity in the numbers of DQN and PPO STAs within the network, the throughput of each STA remains nearly identical. As the time approaches approximately $25\ \SI{}{\second}$, the total throughput stabilizes at around 98\%. This observation indicates that the QPMIX algorithm is capable of training agents to cooperate fairly with different combinations of DQN and PPO agents.

\stepcounter{scenario}
\subsection{Scenario \Rmnum{4}: Coexistence with CSMA/CA STAs}
In this subsection, we evaluate the network performance when the learning-based STAs and conventional CSMA/CA STAs coexist. Specifically, we consider a network with four STAs and each STA is under Poisson traffic with packet arrival rate $\lambda=200\ \text{packets}/ \SI{}{\second}$. Table~\ref{coexistence} shows the network performance when the number of learning-based STAs varies from 0 to 3.
The simulation results indicate that the throughput, mean delay, delay jitter and collision rate are all significantly improved with an increasing number of learning-based STAs, which proves that they can coexist well with traditional STAs. Furthermore, the JFI in different scenarios are close to 0.999, proving the fairness among STAs.

{\begin{table}[!t]
\scriptsize
\setlength{\tabcolsep}{3.8pt}
\renewcommand{\arraystretch}{1}
\caption{Performance when QPMIX coexists with CSMA/CA}
\label{coexistence}
\centering
\begin{tabular}{|>{\centering\arraybackslash}c|c|c|c|c|}
\hline 
\multirow{2}{*}[-2pt]{\footnotesize\centering \textbf{Metric}} & 
\multicolumn{4}{c|}{\footnotesize \textbf{Learning-based STAs number}} \\
\cline{2-5}
& \footnotesize 0 & \footnotesize 1 & \footnotesize 2 & \footnotesize 3 \\
\hline
Throughput & 0.677 & 0.723 & 0.743 & 0.758 \\
\hline
Mean delay (s) & 0.0216 & 0.0087 & 0.0068 & 0.0041 \\
\hline
Delay jitter ($\text{s}^2$) & 9.46 $\times 10^{-5}$ & 7.33 $\times 10^{-5}$ & 5.70 $\times 10^{-5}$ & 2.38 $ \times 10^{-5}$ \\
\hline
Collision rate & 0.134 & 0.086 & 0.061 & 0.049 \\
\hline
JFI & 0.9994 & 0.9997 & 0.9993 & 0.9989 \\
\hline
\end{tabular}
\end{table}
}



\begin{figure*}[!ht]
\begin{minipage}[t]{0.5\textwidth}
\subfigure[QPMIX with 3 DQN and 1 PPO.]{
\includegraphics[width=3.2 in ]{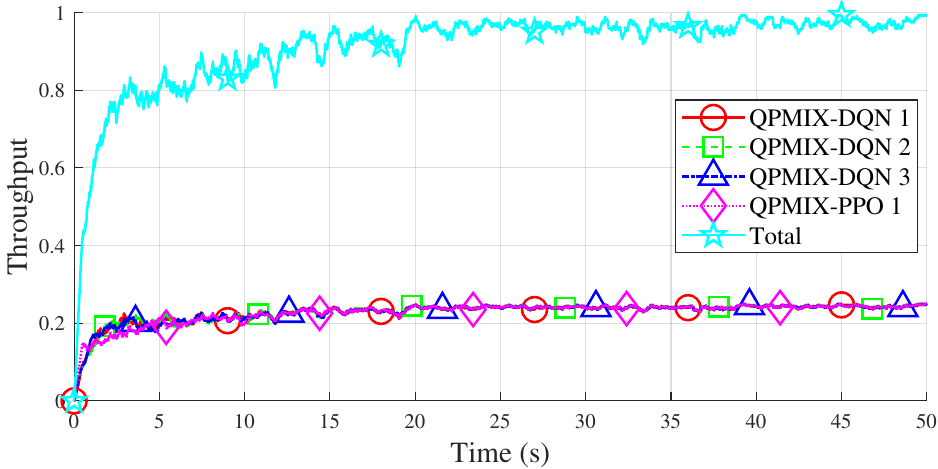}
\label{QPMIX with 3 DQN and 1 PPO}
}
\end{minipage}
\begin{minipage}[t]{0.5\textwidth}
\subfigure[QPMIX with 1 DQN and 3 PPO.]{\includegraphics[width=3.2 in]{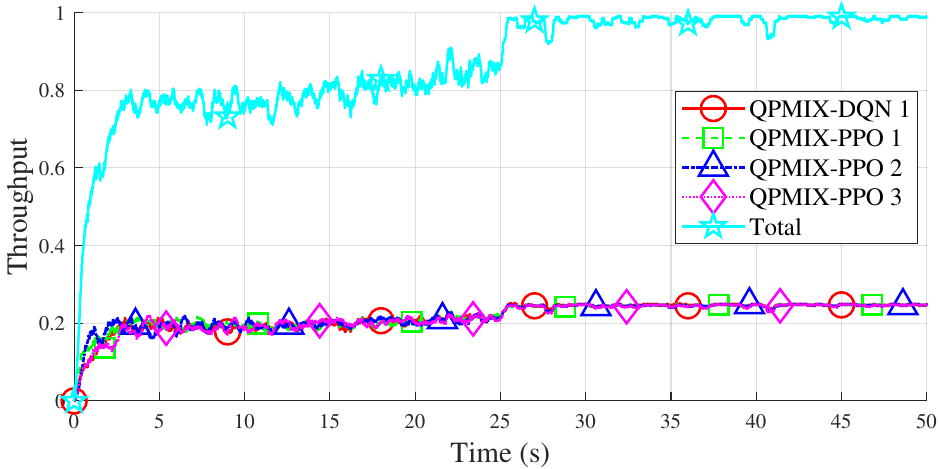}
\label{QPMIX with 1 DQN and 3 PPO}
}
\end{minipage}
\caption{Real-throughput comparison in the different mixings of value and policy-based agents.}
\label{asymmetric_net}
\end{figure*}





\subsection{Independent Learning and Convergence}
\begin{figure*}[htbp]
\begin{minipage}[t]{0.5\textwidth}
\centering
\subfigure[Real-time throughput of QPMIX when $N = 5$.]{\includegraphics[width=3.5in]{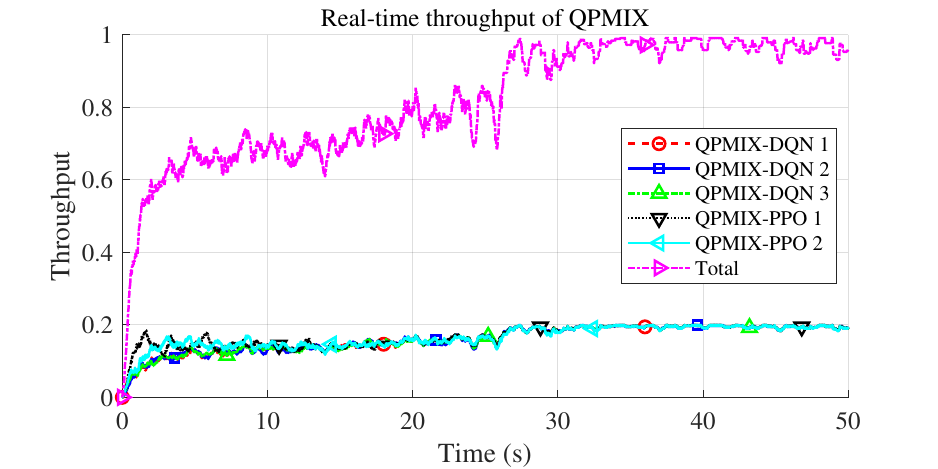}
\label{ablation(a)}
}
\end{minipage}
\begin{minipage}[t]{0.5\textwidth}
\centering
\subfigure[Average reward of independent learning when $N = 5$.]{\includegraphics[width=3.5in]{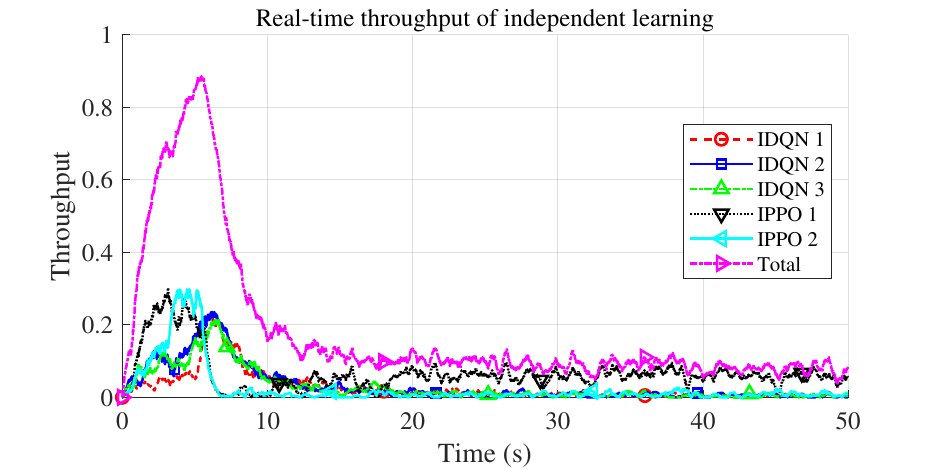}
\label{ablation(b)}
}
\end{minipage}
\caption{Real-time throughput comparison between the QPMIX and independent learning when STAs number is 5.}
\label{ablation_th}
\end{figure*}

\begin{figure}[!t]
\centering
\includegraphics[width=3.5in]{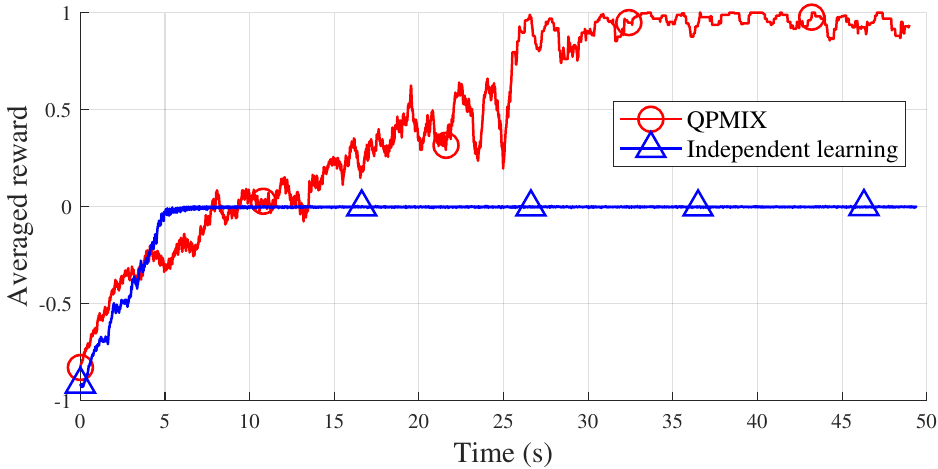}
\caption{Average reward comparison between the QPMIX and independent learning when STAs number is 5.}
\label{ablation_re}
\end{figure}

We conduct an experiment involving independent learning to demonstrate the effectiveness of the proposed QPMIX framework. Independent learning follows a distributed training with distributed execution paradigm, where agents train their own networks based solely on their own experience, with no information interaction among them. This means that there is no mixing network for training in Fig.~\ref{QPMIX}\subref{TrainingFramework}. The reward function and agent networks are the same as the design in QPMIX. Fig.~\ref{ablation_th} and Fig.~\ref{ablation_re} illustrate the real-time throughput and the average reward over the past 500 time steps during training of both the QPMIX algorithm and independent learning under saturated traffic in the Five-STA case. Independent DQN and independent PPO are marked as IDQN and IPPO, respectively. 
From Fig.~\ref{ablation_th}, QPMIX achieves close to 98\% total throughput during training while ensuring fairness. In contrast, the independent learning method shows only a temporary increase in throughput in the early training stages due to agent exploration.
However, the agents eventually converge to a waiting strategy, causing a steady decline in the throughput performance.
We further present in Fig.~\ref{ablation_re} the rewards of the proposed QPMIX algorithm and the independent learning approach as training progresses. 
We average every 500 reward values for the ease of exposition. From Fig.~\ref{ablation_re}, the average rewards of the QPMIX algorithm improve as training progresses, demonstrating the effectiveness of the proposed algorithm. The performance gradually converges to approximately $30\ \SI{}{\second}$ for five STAs. 
However, the average rewards of the independent learning approach converge to zero, indicating that the agents have converged to the Wait action. 
As the policies of other agents dynamically during the learning process, independent learners face a non-stationary environment, making them challenging to cooperate with one another and leading to a suboptimal solution. 
The QPMIX algorithm incorporates a mixing network that implicitly shares global information among agents, facilitating collaborative behaviors and enhancing overall performance. Although the proposed QPMIX algorithm increases the training time compared with the CSMA/CA mechanism, the improvements in throughput, mean delay, delay jitter and collision rate make it a compelling alternative. The additional training time is acceptable, especially considering that the advancements in AI technology allow for the utilization of pre-trained models, which can significantly accelerate the training through online fine-tuning. In summary, the proposed QPMIX algorithm not only encourages cooperative behaviors among value-based and policy-based agents, but also prevents converging to local suboptimal solutions.

\section{Conclusion and Future Work}
In this paper, we have introduced a distributed channel access scheme based on heterogeneous MARL for next-generation WLANs. We have proposed a novel training framework, named QPMIX, which adopts a CTDE paradigm to address the challenges inherent in heterogeneous MARL. Additionally, we have established that heterogeneous MARL can ensure convergence when employing the linear value function approximation.
Simulation results indicated that QPMIX outperforms CSMA/CA in throughput, mean delay, delay jitter, and collision rates performance. Furthermore, it exhibits robustness in both unsaturated and delay-sensitive traffic scenarios, even when trained in saturated traffic conditions. Moreover, QPMIX is effective in encouraging the cooperation among the heterogeneous agents, compared to independent learning. Future work will include designing heterogeneous MARL algorithms under hidden node scenarios.
Extension of the proposed QPMIX framework to scenarios with overlapping BSSs is also an interesting direction worth further investigation.


\begin{appendices}
\section{Proof of Theorem 1}
\label{Appendix A}
Two projection operators $\mathcal{J}: \mathbb{R}^{KN}\to\mathbb{R}^{KN}$ and $\mathcal{J}_{\perp}:\mathbb{R}^{KN}\to\mathbb{R}^{KN}$ are defined to project the parameter $\omega_t$ to the consensus subspace and the disagreement subspace, where $\omega_t=[(\omega_t^1)^\top,\cdots,(\omega_t^N)^\top]^\top\in\mathbb{R}^{KN}.$ Specifically, the projection operation $\mathcal{J}$ and $\mathcal{J}_{\perp}$ are defined as follows:
\begin{equation}
\begin{cases}
\mathcal{J}\omega=\mathbf{1}\otimes\overline{\omega},\\
\mathcal{J}_\perp\omega=\omega-\mathbf{1}\otimes\overline{\omega} = \omega_\perp,
\end{cases}
\end{equation}
where $\otimes$ denotes the Kronecker product, $\mathbf{1} \in \mathbb{R}^{N}$ is a vector with all elements equal to one, and $\overline{\omega}=N^{-1}\sum_{i}^{N}\omega^i$. 
The concrete proof process is divided into two steps. First, we prove that $\lim_{t\to\infty}\omega_{\perp,t}=0$, which means $\lim_{t\to\infty}\omega^i_t=\lim_{t\to\infty}\overline{\omega_t}$. Then, we demonstrate the convergence of the consensus subspace of the agents, which implies $\lim_{t\to\infty}\overline{\omega_t}=\omega_\Theta$. 

\begin{lemma} \label{lemma1}
Under Assumptions \ref{assump1}-\ref{assump3}, the parameters of the Q-function $\omega_t$ are stable, i.e., $\sup_{t\to\infty}\begin{Vmatrix}\omega_t\end{Vmatrix}<\infty$.
\end{lemma}

\begin{lemma}
\label{lemma2}
Under Assumptions \ref{assump1}-\ref{assump3} and Lemma \ref{lemma1}, for any $M > 0$, we have 
\begin{equation}
\sup_{t\to\infty}\mathbb{E}\big(\|\beta_{\omega,t}^{-1}\omega_{\perp,t}\|^2\mathbb{I}_{\{\sup_{t\to\infty}\|\omega_t\|\leq M\}}\big)<\infty,
\end{equation}
where $\mathbb{I_{\{\cdot\}}}$ is the indicator function.
\end{lemma}

\textbf{Step 1. Convergence of $\omega_{\perp,t}\xrightarrow{a.s.}0$.} 

Under Lemma \ref{lemma2} and Assumption \ref{assump3}, for any $M>0$, there exists a constant $G_1$ such that
\begin{equation}
\sum_t\mathbb{E}\left(\|\omega_{\perp,t}\|^2\mathbb{I}_{\{\sup_{t\to\infty}\|\omega_t\|\leq M\}}\right)<G_1\cdot\sum_t\beta^2_{\omega,t}<\infty,
\end{equation}
which implies $\lim_{t\to\infty}\omega_{\perp,t}\mathbb{I}_{\{\sup_{t\to\infty}\|\omega_t\|\leq M\}}=0$. According to Lemma \ref{lemma1}, $\{\sup_t\|\omega_t\|<\infty\}$ holds a.s, which means $\lim_{t\to\infty}\omega_{\perp,t} =0$.

\textbf{Step 2. Convergence of $\overline{\omega_{t}}\xrightarrow{a.s.}\omega_\Theta$.} 

The update of $\omega^i_t$ in (\ref{critic step}) can be rewrite in a compact form as 
\begin{equation}
\begin{cases}
\omega_{t+1}=\omega_{t}+\beta_{\omega,t}\cdot y_{t+1},\\
y_{t+1}=(\hat{\delta}_t^1\phi_t^\top,\cdots,\hat{\delta}_t^N\phi_t^\top)^\top \in \mathbb{R}^{KN}.    
\end{cases}
\end{equation}
Projecting $\omega_{t}$ into the consensus subspace, the iteration of $\overline{\omega_{t}}$ has the form
\begin{equation}
\label{omegaiteration}
\begin{aligned}
\overline{\omega_{t+1}}&=\overline{\omega_{t}}+\beta_{\omega,t}\left(\overline{y_{t+1}}+\beta^{-1}_{\omega,t}\overline{\omega_{\perp,t}}\right),\\
&=\overline{\omega_{t}}+\beta_{\omega,t}\cdot \mathbb{E}\left(\overline{\delta_t}\phi_t|\mathcal{F}_{t,1}\right)+\beta_{\omega,t}\cdot\xi_{t+1},
\end{aligned}
\end{equation}
where $\mathcal{F}_{t,1}=\sigma(r_t,s_t,a_t,\omega_t;t\leq t)$ is an increasing $\sigma$-algebra over time $t$ and $\xi_{t+1}=(\overline{y_{t+1}}+\beta_{\omega,t}^{-1}\overline{\omega_{\perp,t}})-\mathbb{E}\left(\overline{\delta_{t}}\phi_{t}|\mathcal{F}_{t,1}\right)$ is a martingale difference sequence since
\begin{equation}
\mathbb{E}\big[(\overline{y_{t+1}}+\beta_{\omega,t}^{-1}\overline{\omega_{\perp,t}})\big|\mathcal{F}_{t,1}\big]=\mathbb{E}\big(\overline{y_{t+1}}\big|\mathcal{F}_{t,1}\big)=\mathbb{E}\big(\overline{\delta_{t}}\phi_{t}\big|\mathcal{F}_{t,1}\big).
\end{equation}
(\ref{omegaiteration}) is a $K$-dimensional stochastic approximation iteration, which we can verify as satisfying the following lemmas.

\begin{lemma}
\label{lemma3}
$\mathbb{E}\left(\overline{\delta_t}\phi_t|\mathcal{F}_{t,1}\right)$ is Lipschitz continuous in $\overline{\omega_t}$.
\end{lemma}

\begin{lemma}
\label{lemma4}
The martingale difference sequence $\xi_{t+1}$ satisfies $\mathbb{E}[\left\|\xi_{t+1}\right\|^2\mid\mathcal{F}_t]\leq G_2(1+\left\|\overline{\omega_t}\right\|^2)$ for some constant $G_2$.
\end{lemma}

The proof of Lemma \ref{lemma1}-\ref{lemma4} will be given in Appendix \ref{Appendix B}.

Under Assumptions \ref{assump1}-\ref{assump3}, Lemmas \ref{lemma1}, \ref{lemma3} and \ref{lemma4}, according to conclusion of Borkar et al. \cite{borkar2009stochastic} and Theorem D.2 in \cite{zhang2018fully}, the behavior of (\ref{omegaiteration}) is related to its corresponding ordinary differential equation (ODE):
\begin{equation}
\begin{aligned}
\dot{\overline{\omega}}&=\sum_{s}\sum_{\boldsymbol{a}}d_\Theta(s)\pi_\Theta(s, \boldsymbol{a})\mathbb{E}\left[\overline{\delta}\phi|\mathcal{F}_{t,1}\right]\\
&=\sum_{s}\sum_{\boldsymbol{a}}d_\Theta(s)\pi_\Theta(s,\boldsymbol{a})\mathbb{E}\left[\big(r_t+ (\gamma\phi_{t+1}-\phi_t)\overline{\omega_t}\big)\phi|\mathcal{F}_{t,1}\right]\\
&=\Phi^\top D_\Theta\left[R+(\gamma P^\pi-I)\Phi\overline{\omega}\right].
\end{aligned}
\end{equation}
Note that $P^\pi$ has a simple eigenvalue of $1$ and its remaining eigenvalues have real parts less than 1 by the Perron-Frobenius theorem and Assumption \ref{assump2} \cite{zhang2018fully, chen2022communication}. Hence, all the real parts of the eigenvalues of $(\gamma P^\pi-I)$ are negative, as are those of $\Phi^\top D_\Theta(\gamma P^\pi-I)\Phi$ since $\Phi$ is assumed to be full column rank. Therefore, the ODE achieves global asymptotic stability, and its equilibrium point $\overline{\omega}$ fulfills
\begin{equation}
\Phi^\top D_\Theta\left[R+(\gamma P^\pi-I)\Phi\overline{\omega}\right]=0.
\end{equation}
This concludes the convergence of $\overline{\omega_{t}}\xrightarrow{a.s.}\omega_\Theta$. With step 1 and 2 established, Theorem \ref{omegatheorem} is proven. 

\section{Proof of Lemmas \ref{lemma1}-\ref{lemma4}}
\label{Appendix B}
\subsection{Proof of Lemma \ref{lemma1}}
Since $\omega_{t+1}^i=\omega_t^i+\beta_{\omega,t}\cdot\hat{\delta}_t^i\cdot\nabla_\omega Q^i_t(\omega_t^i)$ and $Q^i_t(\omega_t^i) = \phi(s,\boldsymbol{a})^\top\omega^i_t$, we have 
\begin{equation}
\begin{aligned}
\omega_{t+1}^i=&\omega_t^i +\beta_{\omega,t} \cdot\hat{\delta}_t^i\cdot \phi(s,\boldsymbol{a}) = \omega_t^i +\beta_{\omega,t}\cdot h(\omega_t).
\end{aligned}
\end{equation}
We define $\dot{x}(t)=h_r\big(x(t)\big)=h(rx(t))/r$, so the ODE for $\omega_t^i$ is $\dot{\omega_t^i}=h_r\big(\omega(t)\big)=\hat{\delta}_t^i\phi(s,\boldsymbol{a})$.
Because 
\begin{small}
\begin{equation}
\begin{aligned}
&h_\infty\big(\omega(t)\big)=\lim_{r\to\infty}h_r\big(\omega(t)\big)\\ &
=N^{-1}[(\gamma\phi_{t+1}-\phi_t)^\top\sum_{k}^{N}c_{t}(i,k)\omega^k_t]\phi_t^\top,
\end{aligned}
\end{equation}
\end{small}
the origin of $\omega(t)$ is an asymptotically stable equilibrium for the ODE. Thus, under Theorem 2.1 in \cite{borkar2000ode}, the Lemma \ref{lemma1} holds.

\subsection{Proof of Lemma \ref{lemma2}}
The iteration of the disagreement subspace is
\begin{small}
\begin{equation}\begin{aligned} & \omega_{\perp,t+1}=(I-\mathcal{J})\omega_{t+1} =(I-\mathcal{J})(\omega_{t}+\beta_{\omega,t}y_{t})\\
& =(I-\mathcal{J})\Big[(\mathbf{1}\otimes\overline{\omega}_{t}+\omega_{\perp,t}+\beta_{\omega,t}y_{t})\Big]\\  
& =(I-\mathcal{J})\biggl[(\omega_{\perp,t}+\beta_{\omega,t}y_{t})\biggr]\\
& =\Big[\big((I-\mathbf{1}\mathbf{1}^{\top}/N)\otimes I\big)(\omega_{\perp,t}+\beta_{\omega,t}y_{t})\Big],\end{aligned}\end{equation}
\end{small}
where $I\in \mathbb{R}^{N\times N}$. Because of the fact that $(A\otimes B)^\top=A^\top\otimes B^\top$ and $(A\otimes B)(C\otimes D)=(AC)\otimes(BD)$, we have
\begin{small}
\begin{equation}\begin{aligned} & \mathbb{E}\bigg[\left\|\beta_{\omega,t+1}^{-1}\omega_{\perp,t+1}\right\|^2|\mathcal{F}_{t,1}\bigg]=\frac{\beta_{\omega,t}^{2}}{\beta_{\omega,t+1}^{2}}\mathbb{E}\bigg[\bigg(\beta_{\omega,t}^{-1}\omega_{\perp,t}+y_{t}\bigg)^{\top}\\&\big((I-\mathbf{1}\mathbf{1}^{\top}/N)\otimes I\big)\bigg(\beta_{\omega,t}^{-1}\omega_{\perp,t}+y_{t}\bigg)\big|\mathcal{F}_{t,1}\bigg]\\  & \leq\frac{\beta_{\omega,t}^{2}}{\beta_{\omega,t+1}^{2}}\rho\mathbb{E}\bigg[\bigg(\beta_{\omega,t}^{-1}\omega_{\perp,t}+y_{t}\bigg)^{\top}\bigg(\beta_{\omega,t}^{-1}\omega_{\perp,t}+y_{t}\bigg)|\mathcal{F}_{t,1}\bigg]\\ & =\frac{\beta_{\omega,t}^{2}}{\beta_{\omega,t+1}^{2}}\rho\Bigg(\left\|\beta_{\omega,t}^{-1}\omega_{\perp,t}\right\|^2+2\mathbb{E}\bigg[\left\langle\beta_{\omega,t}^{-1}\omega_{\perp,t},y_{t}\right\rangle|\mathcal{F}_{t,1}\bigg]+ \\&
\mathbb{E}[\left\|y_{t}\right\|^2|\mathcal{F}_{t,1}]\Bigg)\\  & 
\leq\frac{\beta_{\omega,t}^{2}}{\beta_{\omega,t+1}^{2}}\rho\Bigg(\left\|\beta_{\omega,t}^{-1}\omega_{\perp,t}\right\|^2+2\left\|\beta_{\omega,t}^{-1}\omega_{\perp,t}\right\|\big[\mathbb{E}(\left\|y_{t}\right\|^2|\mathcal{F}_{t,1})\big]^{\frac{1}{2}}+\\ & \mathbb{E}[\left\|y_{t}\right\|^2|\mathcal{F}_{t,1}]\Bigg),\end{aligned}\end{equation}
\end{small}
where $\rho=1$ is the spectral norm of $\big((I-\mathbf{1}\mathbf{1}^{\top}/N)\otimes I\big)$ because the maximum eigenvalue of $(I-\mathbf{1}\mathbf{1}^{\top}/N)$ is 1, and $\|A\otimes B\|_2=\|A\|_2\cdot\|B\|_2$. Since the rewards $r_t$ and the feature $\phi_t$ are uniformly bounded, $\mathbb{E}[\left\|y_{t}\right\|^2|\mathcal{F}_{t,1}]$ is bounded. Let $\upsilon_t=\left\|\beta_{\omega,t}^{-1}\omega_{\perp,t}\right\|^2\mathbb{I}_{\{\sup_{t\to\infty}\|\omega_t\|\leq M\}}$. Since $\lim_{t\to\infty}\beta_{\omega,t+1}\beta_{\omega,t}^{-1}=1$, there exists a large enough $t_0$ such that for any $t>t_0$, $\frac{\beta_{\omega,t}^{2}}{\beta_{\omega,t+1}^{2}}\leq 1 + \mu$, where $\mu>0$. Hence, there exists a positive constant $H$ such that for any $t>t_0$, 
\begin{equation}
\begin{aligned}
\mathbb{E}(\upsilon_{t+1})&\leq(1+\mu)\big[\mathbb{E}(\mu_t + 2\sqrt{\mathbb{E}(\upsilon_t)\cdot H} + H\big] \\
&=(1+\mu)\big[\sqrt{\mathbb{E}(\upsilon_t)} + \sqrt{H}\big]^2 \\
&\leq 2(1+\mu)(\mathbb{E}(\upsilon_t) + H).
\end{aligned}
\end{equation}
By induction, we can obtain that $\mbox{\(\displaystyle\mathbb{E}(\upsilon_{t+1})\leq [2(1+\mu)]^{t-t_0}(\mathbb{E}(\upsilon_{t_0}) + H)\)}$. Therefore, we obtain $\sup_{t\to\infty}\mathbb{E}\big(\|\beta_{\omega,t}^{-1}\omega_{\perp,t}\|^2\mathbb{I}_{\{\sup_{t\to\infty}\|\omega_t\|\leq M\}}\big)<\infty,$ which concludes the proof.

\subsection{Proof of Lemma \ref{lemma3}}
Lemma \ref{lemma3} holds evidently because 
\begin{small}
\begin{equation}
\begin{aligned}
\overline{\delta_t} &= N^{-1}\sum^N_i\left[r_t+N^{-1}\sum^N_kc_{t}(i,k)(\gamma\phi_{t+1}-\phi_t)^\top\omega^k_t)\right]\\[0.1ex]
&=r_t + (\gamma\phi_{t+1}-\phi_t)^\top N^{-2}\sum^N_k\omega^k_t\sum^N_ic_{t}(i,k)\\[0.1ex]
&=r_t+ (\gamma\phi_{t+1}-\phi_t)^\top N^{-2}\sum^N_k\omega^k_t\\[0.1ex]
&=r_t+ N^{-1}(\gamma\phi_{t+1}-\phi_t)^\top \overline{\omega_t}.
\end{aligned}
\end{equation}
\end{small}

\subsection{Proof of Lemma \ref{lemma4}}
Since $\xi_{t+1}=(\overline{y_{t+1}}+\beta_{\omega,t}^{-1}\overline{\omega_{\perp,t}})-\mathbb{E}\left(\overline{\delta_{t}}\phi_{t}|\mathcal{F}_{t,1}\right)$, we have
\begin{small}
\begin{equation}
\begin{aligned}
\label{Prooflemma1.1}
\mathbb{E}[\left\|\xi_{t+1}\right\|^2\mid\mathcal{F}_t]&\leq2\mathbb{E}\left(\|\overline{y_{t+1}}+\beta_{\omega,t}^{-1}\overline{\omega_{\perp,t}}\|^2\mid\mathcal{F}_t\right)\\
&+2\|\mathbb{E}\big(\overline{\delta_t}\phi_t|\mathcal{F}_{t,1}\big)\|^2.
\end{aligned}
\end{equation}
\end{small}
Because of the fact that $(A\otimes B)^\top=A^\top\otimes B^\top$, we then have 
\begin{small}
\begin{equation}
\begin{aligned}
&\|\overline{y_t}+\beta_{\omega,t}^{-1}\overline{\omega_{\perp,t}}\|^2\\
&=\frac{1}{N^2}\big[(\mathbf{1}^\top\otimes I)(y_t+\beta_{\omega,t}^{-1}\omega_{\perp,t})\big]^\top\big[(\mathbf{1}^\top\otimes I)(y_t+\beta_{\omega,t}^{-1}\omega_{\perp,t})\big]\\
&=\frac{1}{N^2}\big[(y_t+\beta_{\omega,t}^{-1}\omega_{\perp,t})^\top(\mathbf{1}\otimes I)(\mathbf{1}^\top\otimes I)(y_t+\beta_{\omega,t}^{-1}\omega_{\perp,t})\big]\\
&=\|\left(y_t+\beta_{\omega,t}^{-1}\omega_{\perp,t}\right)]\|_{G_t}^2,
\end{aligned}
\end{equation}
\end{small}
where $G_t=\frac{1}{N^2}(\mathbf{1}\otimes I)(\mathbf{1}^\top\otimes I)$. Under Lemma \ref{lemma1}, the first term of (\ref{Prooflemma1.1}) can be bounded over the set $\{\sup_t\|\omega_t\|<K_1\}$, for a constant $K_1>0$. Therefore, there exists $K_2, K_3<\infty$ that 

\begin{equation}
\label{leq1}
\begin{aligned}
&\mathbb{E}\big(\left\|y_{t+1}+\beta_{\omega,t}^{-1}\omega_{\perp,t}\right\|_{\mathrm{G}_{t}}^{2}\big|\mathcal{F}_{t,1}\big)\cdot\mathbb{I}_{\{\sup_{t}\|z_{t}\|\leq K_1\}}\\
&\leq K_{2}\mathbb{E}\big(\left\|y_{t+1}\right\|^{2}+\left\|\beta_{\omega,t}^{-1}\omega_{\perp,t}\right\|^{2}\big|\mathcal{F}_{t,1}\big)\cdot\mathbb{I}_{\{\sup_{t}\|z_{t}\|\leq K_1\}}\\
&<K_{3}.
\end{aligned}
\end{equation}
Moreover, under Lemma \ref{lemma3}, there exists $K_4<\infty$ that
\begin{equation}
\label{leq2}
\|\mathbb{E}\big(\overline{\delta_t}\phi_t|\mathcal{F}_{t,1}\big)\|^2 \leq K_4 (1+\left\|\overline{\omega_t}\right\|^2).
\end{equation}
By applying (\ref{leq1}) and (\ref{leq2}) into (\ref{Prooflemma1.1}), we have 
\begin{equation}
\mathbb{E}[\left\|\xi_{t+1}\right\|^2\mid\mathcal{F}_t]\leq G_2(1+\left\|\overline{\omega_t}\right\|^2).
\end{equation}
This completes the proof.

\section{Proof of Theorem 2}
\label{Appendix C}
Since the Q-functions converge under Theorem \ref{omegatheorem}, the proof of Theorem \ref{thetatheorem} is the same as the techniques from Appendix B in \cite{zhang2018fully}. We first define
\begin{equation}
\zeta_{t+1,1}^j= A_{t}^{j}\cdot\psi_{t}^{j}-\mathbb{E}_{s_{t}\sim d_{\theta_{t}},a^j_{t}\sim\pi_{\theta_{t}}}\big(A_{t}^{j}\cdot\psi_{t}^{j}\big|\mathcal{F}_{t,2}\big)
\end{equation}
and 
\begin{equation}
 \zeta_{t+1,2}^j= \mathbb{E}_{s_{t}\sim d_{\theta_{t}},a^j_{t}}[(A_{t}^{j}-A^j_{t,\theta})\psi_t^j|\mathcal{F}_{t,2}],   
\end{equation}
where $A_{t,\theta}^j$ is the advantage function when $\omega^j_t$ converges to $\omega_\Theta$, i.e., 
$A_{t, \theta}^{j}=Q_{p}^{j}(\omega_\Theta)-\sum_{a^{j}\in\mathcal{A}}\pi^{j}_t(s_t, a_t^{j};\theta_{t}^{j})\cdot Q_{t}^{j}(s_{t},a_{t}^{j},a_{t}^{-j};\omega_\Theta)$ 
and $\mathcal{F}_{t,2}=\sigma(\theta_\tau,\tau\leq t)$ is the $\sigma$-field generated by $\{\theta_{\tau},\tau\leq t\}$. Then the actor update step in (\ref{equation 2}) with a local can be rewritten as 
\begin{small}
\begin{equation}
\begin{aligned}
    &\theta_{t+1}^j=\Gamma^j\Big[\theta_t^j+\beta_{\theta,t}\\
    &\Big(\mathbb{E}_{s_t\thicksim d_{\theta_t},a_t\thicksim\pi_{\theta^j_t}}\big(A_{t,\theta_t}^i\psi_t^j\big|\mathcal{F}_{t,2}\big)+\zeta_{t+1,1}^j+\zeta_{t+1,2}^j\Big)\Big].
\end{aligned}
\end{equation}
\end{small}
Since $\omega^j_t$ converges to $\omega_\Theta$ at the faster step, the sequence $\{\zeta_{t+1,2}^j\}$ converges to zero as $t\rightarrow\infty$. Furthermore, $\{Z^j_t=\sum_{\tau=0}^t\beta_{\theta,\tau}\zeta_{\tau+1,1}^i\}$ is a martingale sequence that satisfies $\sum_t\mathbb{E}\big(\left\|Z_{t+1}^j-Z_t^j\right\|^2\big|\mathcal{F}_{t,2}\big)=\sum_{t\geq1}\left\|\beta_{\theta,t}\zeta_{t+1,1}^i\right\|^2<\infty $ under Assumption~\ref{assump3}, which conducts that the martingale sequence $\{Z^j_t\}$ converges. Thus, for any $K_5 > 0$, we have $\lim\limits_{t\to\infty} \mathbb{P}\bigg(\sup_{n\geq t}\bigg\Vert\sum_{\tau=t}^n\beta_{\theta,\tau}\zeta_{\tau,1}^i\bigg\Vert\geq K_5\bigg)=0.$ Moreover, $\mathbb{E}_{s_t\thicksim d_{\theta_t},a_t\thicksim\pi_{\theta^j_t}}\big(A_{t,\theta_t}^i\psi_t^j\big|\mathcal{F}_{t,2}\big)$ is continuous in $\theta^j_t$ under Assumption \ref{assump2}. Therefore, the update in (\ref{equation 2}) converges to the set of asymptotically stable equilibria of the ODE (\ref{Theorem2}) under Kushner-Clark lemma \cite{zhang2018fully}.
\end{appendices}

\bibliography{bare_jrnl_new_sample4}
\bibliographystyle{IEEEtran}





\vfill

\end{document}